\documentclass[letterpaper, 10 pt, conference]{ieeeconf}  

\IEEEoverridecommandlockouts                              

\usepackage{fancyhdr}

\usepackage{xcolor}
\usepackage{comment}

\usepackage[noadjust]{cite}
\usepackage{url}

\usepackage[mathcal]{euscript}

\usepackage[cmex10]{amsmath}
\usepackage{amssymb}

\usepackage{amsthm} 
  
\usepackage{dsfont} 

\usepackage{epsfig} 
\usepackage{tikz}

\usepackage{multirow}
\usepackage{enumerate}

\usepackage[ruled, vlined, linesnumbered]{algorithm2e}

\newcommand{\rewrite}[1] 	{{\color{red}#1}} 

\newtheoremstyle{plain}
	  {}
	  {}
	  {\itshape}
	  {}
	  {\bfseries}
	  {}
	  {5pt plus 1pt minus 1pt}
	  {}

\newtheoremstyle{definition}
  	  {}
	  {}
	  {\normalfont}
	  {}
	  {\bfseries}
	  {}
	  {5pt plus 1pt minus 1pt}
	  {}
  
\theoremstyle{plain}

\newtheorem{proposition}{Proposition}

\theoremstyle{definition}
\newtheorem{definition}{Definition}
\newtheorem{example}{Example}
\newtheorem{remark}{Remark}

\newcommand{\refeq}[1]			{(\ref{#1})} 
\newcommand{\reffig}[1]			{Fig. \ref{#1}} 
\newcommand{\refsec}[1]			{Section \ref{#1}}

\newcommand{\refprop}[1]		{Proposition \ref{#1}}

\newcommand{\refrem}[1]			{Remark \ref{#1}}

\newcommand{\reffn}[1] 		    {\textsuperscript{\ref{#1}}}

\newcommand{\state}{\vect{x}}
\newcommand{\ctrl}{\vect{u}}
\newcommand{\youtput}{\vect{y}}
\newcommand{\goal}{\state^{*}}
\newcommand{\goalctrl}{k_{\goal}}
\newcommand{\barrierfunc}{h}
\newcommand{\barrierrate}{\alpha}
\newcommand{\barriercorridor}{\mathrm{BC}}
\newcommand{\smin}{\mathrm{smin}}
\newcommand{\softness}{\lambda}
\newcommand{\ctrlgain}{\kappa}
\newcommand{\nstate}{n_{\vect{x}}}
\newcommand{\nctrl}{n_{\vect{u}}}
\newcommand{\noutput}{n_{\vect{y}}}
\newcommand{\nbarrier}{m}
\newcommand{\refpath}{\mathrm{p}}
\newcommand{\refpathgoal}{\mathrm{p}^*}

\newcommand{\pos}{\vect{x}}
\newcommand{\orient}{\theta}
\newcommand{\goalpos}{\vect{x}^*}
\newcommand{\linvel}{v}
\newcommand{\angvel}{\omega}
\newcommand{\lingain}{\kappa_{v}}
\newcommand{\anggain}{\kappa_{\omega}}
\newcommand{\ovect}[1]{\begin{bmatrix}\cos #1 \\ \sin #1 \end{bmatrix}}
\newcommand{\nvect}[1]{\begin{bmatrix} -\sin #1 \\ \cos #1\end{bmatrix}}
\newcommand{\ovectsmall}[1]{\scalebox{0.8}{$\begin{bmatrix}\cos #1 \\ \sin #1 \end{bmatrix}$}}
\newcommand{\nvectsmall}[1]{\scalebox{0.8}{$\begin{bmatrix} -\sin #1 \\ \cos #1\end{bmatrix}$}}

\newcommand{\R}  	{\mathbb{R}} 

\newcommand{\radius} 	{r}

\let\originalleft\left
\let\originalright\right
\renewcommand{\left}{\mathopen{}\mathclose\bgroup\originalleft}
\renewcommand{\right}{\aftergroup\egroup\originalright}
	
\newcommand{\plist}[1] 	{\left(#1\right)} 
\newcommand{\clist}[1]	{\left\{#1\right\}} 

\newcommand{\vect}[1]   {\mathrm{#1}}
\newcommand{\mat}[1]    {\mathbf{#1}}
\newcommand{\tr}[1] {{#1}^{\mathrm{T}}} 
\newcommand{\norm}[1]  {\|#1\|}

\newcommand{\argmin}{\operatornamewithlimits{arg\ min}} 
\newcommand{\argmax}{\operatornamewithlimits{arg\ max}} 
\newcommand{\diff} {\mathrm{d}}

\title{\LARGE \bf
Control Barrier Corridors: From Safety Functions to Safe Sets
\\{\vspace{-1mm}}
{\normalsize (Extended Version)} 
\\{\vspace{-3mm}}
}

\author{\"{O}m\"{u}r Arslan and Nikolay Atanasov
 \thanks{\"{O}m\"{u}r Arslan is with the Department of Mechanical Engineering, Eindhoven University of Technology, P.O. Box 513, 5600 MB Eindhoven, The Netherlands.  
 Email: o.arslan@tue.nl}%
 \thanks{Nikolay Atanasov is with the Department of Electrical and Computer Engineering, University of California San Diego, La Jolla, CA 92093, USA. 
 Email: natanasov@ucsd.edu}%
 }

\begin{document}

\maketitle
\thispagestyle{empty}
\pagestyle{empty}

\begin{abstract}
Safe autonomy is a critical requirement and a key enabler for robots to operate safely in unstructured complex environments. 
Control barrier functions and safe motion corridors are two widely used but technically distinct safety methods, functional and geometric, respectively, for safe motion planning and control.
Control barrier functions are applied to the safety filtering of control inputs to limit the decay rate of system safety, whereas safe motion corridors are geometrically constructed to define a local safe zone around the system state for use in motion optimization and reference-governor design.
This paper introduces a new notion of \emph{control barrier corridors}, which unifies these two approaches by converting control barrier functions into local safe goal regions for reference goal selection in feedback control systems.
We show, with examples on fully actuated systems, kinematic unicycles, and linear output regulation systems, that \emph{individual state safety can be extended locally over control barrier corridors} for convex barrier functions, provided the control convergence rate matches the barrier decay rate, highlighting a trade-off between safety and reactiveness.
Such safe control barrier corridors enable safely reachable persistent goal selection over continuously changing barrier corridors during system motion, which we demonstrate for verifiably safe and persistent path following in autonomous exploration of unknown environments.
\end{abstract}

\section{Introduction}
\label{sec.introduction}

Ensuring that autonomous systems operate safely while achieving performance and stability objectives has emerged as a central challenge in robotics and automation. 
Control barrier functions (CBFs) \cite{ames_etal_ECC2019} have recently gained significant traction as a principled tool for enforcing safety by modifying control inputs to bound the safety decay rate. 
CBFs have also been combined with control Lyapunov functions (CLFs) \cite{artstein1983stabilization, sontag1989universal} to reconcile joint stability and safety requirements \cite{jankovic_ACC2017}. 
In parallel, geometrically constructed safe motion corridors have been applied for safe reference planning \cite{marcucci_etal_TRO2024} and combined with reference governor techniques \cite{liu_etal_RAL2017, kolmanovsky_garone_dicairano_ACC2014} to achieve safety guarantees without interfering with stability \cite{arslan_koditschek_ICRA2017}.
Understanding the connections among CBFs, reference governor techniques, and safe motion corridors remains relatively unexplored and has the potential to offer new insights, combine their strengths, and mitigate their limitations (e.g., the proper selection of the CBF safety decay rate relative to control convergence and the systematic extension of individual state safety to local safety zones for complex control systems) \cite{liang_etal_ACC2024, freire_debarshi_nicotra_CSL2025}.

\begin{figure}[t]
\centering
\begin{tabular}{@{}c@{}}
\hline
\hline
\footnotesize{From Control Barrier Functions To Control Barrier Corridors}
\\
\hline
\\ [-3mm]
\scalebox{0.69}{\hspace{-3mm}
$
\left .
\begin{array}{l}
\bullet \text{ Barrier Safety: } \\
\hspace{4mm} h_i(\state) \geq 0  \quad \forall i\!=\!1,...,\nbarrier\\
\bullet \text{ Safety Feasibility: }\\
\hspace{4mm} \dot{h}_{i}(\state) \geq - \alpha (h_i(\state)) \quad \forall i \\
\bullet \text{ System Dynamics: }  \\
\hspace{4mm} \dot{\state} = f(\state, \ctrl), \\
\bullet \text{ Goal Control: } \\
\hspace{4mm} \ctrl = k_{\goal}(\state)
\end{array}
\hspace{-2.0mm}
\right \} 
\hspace{-1mm} 
\begin{array}{l}
\bullet \text{ Generic Barrier Corridor: }\\ 
\quad \!  \mathrm{BC}(\vect{x})\!=\! \bigcap\limits_{i=1}^{\nbarrier}\!\clist{\goal \!\!\in  \! \R^{n} \Big |  \tr{\nabla h_i(\state)\!} \!f(\state, k_{\goal}(\state)\!) \!\geq\! - \barrierrate(h_i(\state)\!)\!}
\\
\\
\bullet \text{ Barrier Corridor for Proportional Control: }\\ 
\quad \!  \barriercorridor(\state) \hspace{-7.5mm} \underset{\substack{\dot{\state} = \ctrl \\ \ctrl = -\ctrlgain(\state - \goal)}}{\Big |} \hspace{-8mm}= \! \bigcap\limits_{i=1}^{\nbarrier}\! \clist{\goal \!\!\in \!\R^{n} \Big |\! -\!\kappa \tr{\nabla h_i(\vect{x})\!} \!(\vect{x} \!-\! \goal\!) \!\geq\! - \barrierrate (h_i(\state)\!)\!}
\end{array}
\hspace{-3mm}
$
}
\\
\\[-3mm]
\hline
\hline
\\[-3mm]
\begin{tabular}{@{}c@{\hspace{1mm}}c@{\hspace{1mm}}c@{}}
\includegraphics[width=0.55\columnwidth]{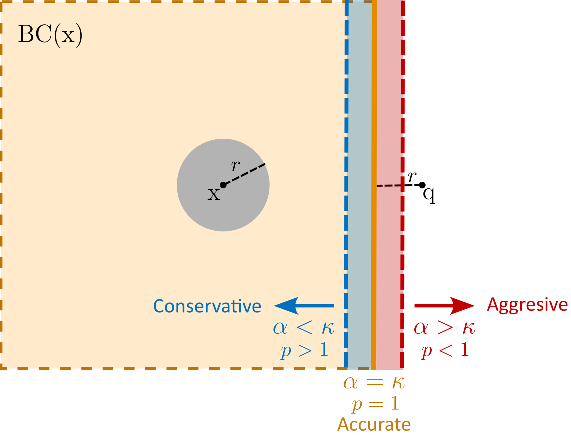}
\raisebox{5.5mm}{\includegraphics[width=0.36\columnwidth, trim=10mm 10mm 10mm 10mm, clip]{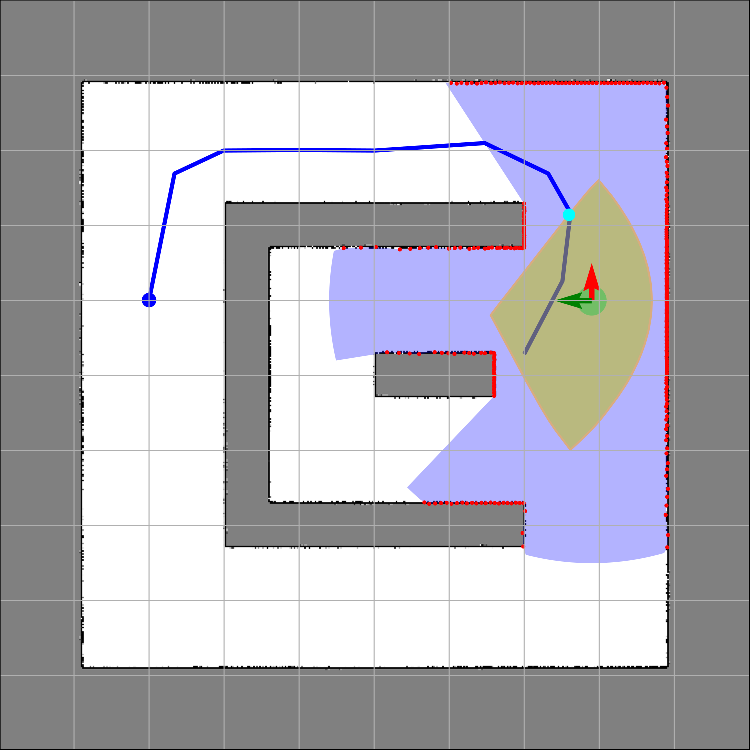}}
\end{tabular}
\end{tabular}
\vspace{-5mm}
\caption{Control barrier corridors turn CBF safety specification on control inputs into a geometric representation of safe goals for state feedback control. For instance, a control barrier corridor $\barriercorridor(\state)$ around a safe state $\vect{x}$ of a fully actuated system $\dot{\vect{x}} = \vect{u}$ under proportional control $\ctrl = -\ctrlgain(\state - \goal)$ is the set of goal states $\goal$ that ensure safe control based on the feasibility condition $\dot{h}_i(\state)\geq -\alpha (h_i(\state))$ for a set of CBFs $h_1(\state), \ldots, h_{\nbarrier}(\state)$.
(left) The control barrier corridor of a fully actuated circular robot, centered at $\state$ with radius $r > 0$, relative to a point obstacle at $\vect{q}$, constructed based on the power-distance barrier function $h(\state) = \norm{\state - \vect{q}}^ p - r^p$. The control barrier corridor enlarges with an increasing ratio of the barrier decay rate to the proportional control gain, $\frac{\barrierrate}{\ctrlgain}$, and shrinks with the increasing order $p$ of the power distance.
(right) The control barrier corridor (yellow) of a fully actuated circular robot (cyan+red arrow), constructed relative to multiple sensed obstacle points (red), captures a local collision-free space around the robot for convex barrier functions (e.g., the power distance with $p \geq 1$) under identical linear barrier decay rate and proportional control gain, $\barrierrate = \ctrlgain$.
This allows for selecting a safely reachable goal (cyan) for following a reference path (blue).}
\label{fig.control_barrier_corridors}
\vspace{-4mm}
\end{figure}

This paper introduces a new unifying approach to transform the functional safety constraints of CBFs into geometric safety restrictions on reference goals for feedback control systems. Specifically, we consider a given stabilizing control law for the system and study the set of goals (closed-loop equilibria) that satisfy CBF safety constraints. We refer to this region of safe closed-loop system goals as a \emph{control barrier corridor}. 
Our key insight is that individual state safety imposed through CBFs can be locally extended into a control barrier corridor under appropriate conditions on the barrier functions, their decay profiles, and the closed-loop system convergence rate, explicitly revealing a trade-off between safety and reactiveness. In particular, we show that, if (i) the barrier functions are convex and (ii) the barrier decay rate matches the closed-loop system convergence rate, then simultaneously safely reachable and stable reference goals can be selected from control barrier corridors.
This enables verifiably safe and persistent path following of a mobile robot by continuously chasing a path goal within the sensed control barrier corridors, as illustrated in \reffig{fig.control_barrier_corridors}.

\subsection{Motivation and Related Literature}

\vspace{-1mm}

\paragraph{Control Barrier Functions} 
Safe filtering of control inputs using CBF constraints has become a de facto tool for enforcing safety for complex control systems. A particularly attractive property is that, for any control-affine system, the CBF constraints are linear in the control input, which enables real-time safe control synthesis via convex quadratic optimization \cite{ames_etal_ECC2019}. While CBF techniques are elegant and efficient, they modify the control inputs directly and, thus, interfere with the stability properties of the closed-loop system. It has been shown that the common strategy of imposing both CBF and CLF constraints and relaxing the stability constraints if safety is endangered leads to undesirable equilibria \cite{reis2021undesirable,Yi_CLFCBFDCP_LCSS23,tan2024undesired}. Avoiding this requires the satisfaction of compatibility conditions \cite{romdlony2016clbf,braun2020comment,mestres2023roa} and potentially modifying the CBF or CLF constraints to ensure that they are compatible \cite{wang2018permissive,braun2019cclf,meng2022converse,dai2024compatibility}.
The performance and conservatism of CBF-based safety filtering also significantly depend on the selection of the barrier decay profiles, and the online adaptation of barrier decay rates can mitigate conservatism \cite{parwana_panagou_ACC2025, kim_kee_panagou_ICRA2025, ong_etal_CDC2025}.
Our analysis provides new connections to recent work on the compatibility of safety and stability \cite{dai2024compatibility, Yi_CLFCBFDCP_LCSS23}, as well as on the proper selection of the barrier decay rate relative to control convergence rate \cite{parwana_panagou_ACC2025, kim_kee_panagou_ICRA2025, ong_etal_CDC2025}.

\paragraph{Reference Governors} Safety filtering of reference goals for pre-stabilized feedback control systems via reference governor techniques \cite{bemporad1998reference, gilbert2002reference, kolmanovsky_garone_dicairano_ACC2014} is attractive because it ensures safety without interfering with stability. 
Explicit reference governors can achieve real-time performance comparable to CBF-based safety filters \cite{garone_nicotra_TAC2016}.
The key idea of reference governors is that rather than altering the feedback law itself, a virtual reference governor system is used to filter the desired reference trajectory so that the resulting closed-loop system remains within the prescribed safe sets. 
This is achieved by determining invariant sets (e.g., Lyapunov level sets \cite{garone_nicotra_TAC2016} and maximal output admissible sets \cite{gilbert_tan_TAC1991}) to determine how quickly the reference can evolve while ensuring future trajectories remain safe. 
Local collision-free zones, constructed as safe motion corridors, are also utilized by designing stabilizing feedback laws that also ensure safe zone invariance \cite{arslan_koditschek_IJRR2019}.
Reference governors have been demonstrated in safe trajectory tracking control for integrator \cite{arslan_koditschek_ICRA2017, Isleyen2022lowtohigh}, linear \cite{li2020directional, burlion2022RGlinear, Li_GCBF_Automatica23}, unicycle \cite{Isleyen2023safeunicycle1,Isleyen2023safeunicycle2,tarshahani_isleyen_arslan_ECC2024}, and more complex systems, such as aerial vehicles or manipulators~\cite{Tartaglione2024uav,Merckaert2022manipulators,Merckaert2024manipulators}.
In this work, we present the use of control barrier corridors as a local safety zone in reference-governor design for the safe path following of a mobile robot with unicycle~dynamics.  

\paragraph{Safe Motion Corridors} Safe motion corridors are also employed as geometrically constructed convex planning and control constraints for safe, smooth, and dynamically feasible path planning \cite{liu_etal_RAL2017, marcucci_etal_TRO2024}, as well as for reactive navigation control via reference governors \cite{arslan_koditschek_IJRR2019}.
Such safe corridors are often constructed as a local collision-free geometric zones around obstacles in various convex shapes, including polytopes \cite{liu_etal_RAL2017, deits_tedrake_AFR2015}, boxes \cite{marcucci_etal_TRO2024, chen_liu_shen_ICRA2016}, and spheres \cite{gao_etal_JFR2019}, by maximizing their volumes to approximately cover the collision-free configuration space of robot systems. 
Safe motion corridors constructed using generalized Voronoi diagrams are applied for safety encoding and safe coordination control of multi-robot systems \cite{zhou_wang_banyopadhyay_schwager_RAL2017, zhu_brito_alonso-mora_AR2022, arslan_koditschek_ICRA2016}.
In fact, safe corridors constructed based on separating hyperplanes of convex shapes (e.g., spheres) \cite{zhou_wang_banyopadhyay_schwager_RAL2017, zhu_brito_alonso-mora_AR2022, arslan_koditschek_ICRA2016} are a special case of control barrier corridors for fully actuated systems, where the barrier functions are given by the Euclidean distance to obstacles, with proportional control gain matched to the barrier decay rate (as in \refprop{prop.geometry_meets_dynamics}).
Therefore, we design control barrier corridors to provide a unified, functionally generated extension and generalization of  geometrically constructed safe motion corridors by employing alternative (e.g., higher-order) barrier functions for complex control systems~\cite{xiao_belta_TAC2022}.

\subsection{Contributions and Organization of the Paper}

This paper introduces a new notion of control barrier corridors that systematically convert the functional safety requirements of control barrier functions into local geometric goal zones for safe reference selection in feedback control systems, blending the geometry and dynamics of safety.
In summary, the three major contributions of our paper are:

\quad $\bullet$ We propose a generic construction of control barrier corridors for feedback control systems using control barrier functions, as the collection of all goal states around a safe system state that ensure the maintenance of barrier safety under feedback control.

\quad $\bullet$ We show that control barrier corridors of fully actuated systems under proportional control encode not only robot safety but also goal safety when the control barrier functions are convex and the proportional control gain matches the linear barrier decay rate.
We also discuss how this extends to unicycle and linear systems, as well as the role of the convexity of control barrier functions on safety.

\quad $\bullet$ We apply control barrier corridors to select safely reachable local goals for sensor-based safe and persistent path following control of a mobile robot in unknown cluttered environments.

\smallskip


The rest of the paper is organized as follows.
\refsec{sec.barrier_functions_to_barrier_corridors} describes how to convert control barrier functions into control barrier corridors.
\refsec{sec.example_control_barrier_corridors} presents example constructions of control barrier corridors and discusses their key properties.
\refsec{sec.safe_path_following} demonstrates an application of control barrier corridors for local goal selection in safe and persistent path following.
\refsec{sec.conclusions} concludes with a summary of our work and key results and outlines directions for future research.

\section{\!\!From Barrier Functions to Barrier Corridors\!\!}
\label{sec.barrier_functions_to_barrier_corridors}

In this section, we review CBFs for safety-critical control systems~\cite{ames_etal_ECC2019} and present the construction of control barrier corridors for goal-parametrized feedback control systems.

\subsection{Control Barrier Functions}
\label{sec.control_barrier_functions}

Consider a nonlinear control system with state $\state \in \R^{n_{\state}}$ and control input $\ctrl \in \R^{n_{\ctrl}}$ and dynamics specified by a Lipschitz continuous function $f: \R^{n_{\state}} \times \R^{n_{\ctrl}} \rightarrow \R^{n_{\state}}$:
\begin{align}
\dot{\state} = f(\state, \ctrl).
\end{align}
Let $\mathcal{X} \subseteq \R^{n_{\state}}$ denote a set of \emph{safe} system states.
Control barrier functions are a widely used tool for safe control design, extending state safety encoded by barrier functions to safe state evolution by limiting the safety decay rate and enforcing forward set invariance at the boundary \cite{blanchini_Automatica1999}.

\begin{definition} \label{def.control_barrier_function}
\emph{(Control Barrier Functions)}
A continuously differentiable function $\barrierfunc\!:\R^{n_{\state}} \!\rightarrow\! \R$ is a \emph{control barrier function} for the system $\dot{\state} = f(\state, \ctrl)$~if
\begin{itemize}
\item (\emph{Safe}) $h(\state) \geq 0$ for all safe system states $\state \in \mathcal{X}$ and negative otherwise\footnote{Since encoding safety with a single barrier function is difficult for complex systems, safe states are often determined through the simultaneous nonnegativity of multiple barrier functions.}.

\item (\emph{Nonsingular}) $\nabla h(\state) \neq 0$ when $h(\vect{x}) = 0$.
    
\item (\emph{Feasible}) For any safe state $\state \in \mathcal{X}$ (i.e., $h(\state) \geq 0$), there exist some control input $\ctrl \in \R^{n_{\ctrl}}$  such that 
\begin{align*}
\dot{h}(\state) = \tr{\nabla h(\state)} f(\state, \ctrl) \geq - \barrierrate(h(\state)) 
\end{align*} 
where $\nabla h(\state)$ denotes the gradient of $h(\state)$ with respect to $\state$, and $\barrierrate : \R \rightarrow \R$ is a strictly increasing continuous function with $\barrierrate(0) = 0$ (i.e., a class $\mathcal{K}_{\infty}$ function).
\end{itemize}
\end{definition}

Intuitively, the nonnegativity of $\barrierfunc(\state)$ implicitly specifies the set membership of safe system states in $\mathcal{X}$; the nonsingularity of the gradient $\nabla \barrierfunc(\state)$ at the critical safety boundary, where $\barrierfunc(\state) = 0$, is required to clearly indicate the direction of increasing safety; and finally, the feasibility of $h(\state)$ ensures that the system can remain safe under some control input, even with a potentially bounded decrease in safety.

Finding a single control barrier function to encode the safety of complex systems (e.g., a mobile robot moving in a cluttered environment with complex-shaped obstacles) is a known challenge. 
Alternatively, multiple simple barrier functions, say $h_1(\state), \ldots, h_{\nbarrier}(\state)$, are often used simultaneously for determining safety of a general nonlinear control system $\dot{\state} = f(\state, \ctrl)$.
They also enable safety filtering of a desired control policy $\ctrl_d(\state)$ for the system $\dot{\state} = f(\state, \ctrl)$ via metric projection and optimization, in order to minimally modify the reference control while ensuring safety feasibility, as  \cite{ames_xu_grizzle_tabuada_TAC2017}
\begin{align}\label{eq.safety_filtering}
\begin{array}{r@{\hspace{2mm}}l}
\underset{\ctrl \, \in \, \R^{n_{\ctrl}}}{\mathrm{minimize}} & \norm{\ctrl - \ctrl_{d}(\state)}^2 \\
\mathrm{subject\, to} & \tr{\nabla h_i(\vect{x})\!} f(\vect{x}, \ctrl)\! \geq \! -\alpha(h_i(\vect{x})\!) \quad\!\! \forall i\!=\!1,\ldots, \nbarrier. \hspace{-6mm} 
\end{array}
\end{align}  
To take advantage of convexity, safety filtering is often applied to control-affine systems of the form $\dot{\state} = f(\state) + g(\state)\ctrl$, so that the optimal safety filtering problem becomes a convex quadratic optimization, which can be solved efficiently using state-of-the-art off-the-shelf solvers \cite{boyd_vandenberghe_ConvexOptimization2004}.
Composing multiple barriers into a single function using their (soft) minimum and product is also used for the explicit analytical solution of optimal safety filters for control-affine systems \cite{molnar_ames_CSL2023}, but this often leads to the relatively unexplored issue of convexity loss in control barrier functions (as highlighted in \refrem{rem.multiple_control_barrier_composition} below).
In this paper, we show that the convexity of control barrier functions is a critical feature for relating and extending the functional safety of individual robot states to a local geometric safety zone around a safe state via control barrier corridors (see \refprop{prop.geometry_meets_dynamics}).
To demonstrate the significance of convexity in control barrier functions, we consider the $p^{\text{th}}$-order power distance as a classical safety measure for circular robots around multiple obstacle points, defined as the $p^{\text{th}}$ power of the Euclidean distance between the robot position $\state$ and an obstacle point $\vect{q}$ as 
\begin{align}
h_{\mathrm{pwr}}(\state) := \norm{\state - \vect{q}}^{p} - r^p,
\end{align}
whose gradient and Hessian are, respectively, given by
\begin{align}
\nabla h_{\mathrm{pwr}}(\state) &= p \norm{\state - \vect{q}}^{p-2} (\state - \vect{q}),
\\
\nabla^2 h_{\mathrm{pwr}}(\state) & = p \norm{\state - \vect{q}}^{p-2} \plist{\mat{I} + (p-2)  \tfrac{(\state - \vect{q})\tr{(\state - \vect{q})}}{\norm{\state - \vect{q}}^2}},
\end{align}
where \mbox{$r > 0$} represents the barrier safety margin, for instance, corresponding to the sum of the robot radius, obstacle radius, and a safety tolerance.
Note that the power-distance barrier $h_{\mathrm{pwr}}(\state)$  is convex for $p \geq 1$ and strictly convex for $p > 1$, with strong convexity holding for $p = 2$.

\subsection{Control Barrier Corridors}
\label{sec.control_barrier_corridors}

Control barrier functions are primarily utilized to determine the safety of individual robot states and the control constraints for maintaining safety, but their connection to local safety around a given robot state remains unexplored.
To study and extend the safety of individual robot states to a local safety zone using control barrier functions and their feasibility, we consider a given goal-parametrized state-feedback control (or goal control, for short) $\ctrl = k_{\goal}(\state)$ that can bring the system $\dot{\state} = f(\state, \ctrl)$ to any goal state $\goal \in \R^{n_{\state}}$.
Accordingly, we define the \emph{control barrier corridor} associated with a set of control barrier functions and a goal control policy of a control system as the set of all goal states that induce safe feedback control for the closed-loop system:
\begin{definition}\label{def.control_barrier_corridors}
(\emph{Control Barrier Corridors})
For a given set of control barrier functions $h_1(\state), \ldots, h_{\nbarrier}(\state)$ and a goal control policy $\ctrl = \goalctrl(\state)$ of the control system $\dot{\state} = f(\state, \ctrl)$, the \emph{control barrier corridor}, denoted by $\barriercorridor(\state)$, around a safe robot state $\state \in \R^{\nstate}$ (with $h_{i}(\state) \geq 0$ for all $i = 1, \ldots, \nbarrier$) of the closed-loop control system $\dot{\state} = f(\state, \goalctrl(\state))$ is defined as the set of all goal states $\goal \in \R^{\nstate}$ that result in a safe control input under the goal control $\ctrl = \goalctrl(\state)$, as 
\begin{align}\label{eq.control_barrier_corridor}
\!\!\!\barriercorridor(\state) \! := \!\bigcap_{i=1}^{\nbarrier} \clist{\goal \!\! \in \!\R^{\nstate} \!\Big | \tr{\nabla h_i(\state)\!}\! f(\state, k_{\goal\!}(\state)\!)  \!\geq\!  - \barrierrate(h_i(\state)\!)\!} \!\!\!
\end{align}
where $\alpha(.)$ is a class-$\mathcal{K}_{\infty}$ barrier decay function. 
\end{definition}

\noindent Hence, instead of directly considering control inputs, we focus on goal states for feedback control to investigate how control barrier functions encode safety geometrically.
Note that if the goal control satisfies \mbox{$f(\state, \goalctrl(\state)) = 0$} for all \mbox{$\state = \goal \in \R^{\nstate}$}, then the barrier corridor $\barriercorridor(\state)$ in \refeq{eq.control_barrier_corridor} is a local neighborhood of the safe robot state $\state$, i.e., $\state \in \barriercorridor(\state)$.
Hence, a natural question is whether, and under what conditions, the goal states in the barrier corridor $\barriercorridor(\state)$ partially or fully share the safety properties of the robot state $\state$.
It is also important to observe that a control barrier corridor $\barriercorridor(\state)$ can, in general, have an arbitrarily complex shape and topological connectivity, depending on the functional properties of the control barrier functions $h_{i}(\state)$, the system dynamics $f(\state, \ctrl)$, and the goal control policy $\goalctrl(\state)$.
Hence, in order to extend individual safety to a local safety zone, we study below some example control barrier corridors of various control systems by drawing parallels with the convexity of control barrier functions, i.e.,
\begin{align*}
h_i(\goal) \geq h_i(\state) + \nabla h_i(\state)(\goal - \state).
\end{align*}  

\section{Examples of Control Barrier Corridors}
\label{sec.example_control_barrier_corridors}

In this section, to gain more insight about the geometry of safety encoded by control barrier functions, we present example constructions of control barrier corridors for fully actuated systems with proportional control, kinematic unicycle systems with inner-outer loop control, and linear time-invariant systems with output regulation control. 
An important observation is that individual state safety can be locally extended if the control barrier functions are convex and the control convergence rate matches the barrier decay~rate.

\subsection{Control Barrier Corridors of Fully Actuated Systems}

As the first example, assuming full-state feedback linearization of control systems, we consider a simple yet nontrivial and informative case: the first-order fully actuated kinematic system model with proportional feedback control,
\begin{align}
\dot{\state} = \ctrl, \quad  \text{and} \quad  \ctrl = -\ctrlgain(\state - \goal),
\end{align}      
where $\ctrlgain > 0$ is a fixed control gain determining the exponential convergence rate of the closed-loop system to any given goal state $\goal \in \R^{\nstate}$. 
For simplicity, we also consider a linear barrier decay rate function $\alpha(x) = \alpha x $, which corresponds to an exponential decay rate for the lower bound of the barrier function value.
Accordingly, for a given set of control barrier functions $h_1(\vect{x}), \ldots, h_{\nbarrier(\state)}$ of the fully actuated system \mbox{$\dot{\state} = \ctrl$}, under proportional goal control \mbox{$\ctrl = -\ctrlgain(\state - \goal)$}, the control barrier corridor $\barriercorridor_{\mathrm{full}}(\state)$ around a safe robot state $\vect{x} \in \R^{\nstate}$ (with $h_i(\vect{x}) \geq 0$ for all $i$) contains all possible goal states $\goal \in \R^{\nstate}$ that ensure safety feasibility and is given by%
\begin{align}\label{eq.full_control_barrier_corridor}
\!\!
\scalebox{0.95}{$\barriercorridor_{\mathrm{full}}(\state) \!=\! \bigcap\limits_{i=1}^{\nbarrier} \clist{\goal \!\! \in\! \R^{\nstate} \!\Big |\! -\! \ctrlgain \tr{\nabla h_i(\state)\!}\! (\state \!-\! \goal \!)\! \geq\! -\barrierrate h_i(\state)\!}$} \!\!\!
\end{align}
where $\kappa >0$ is the proportional control gain and $\alpha >0$ is the linear barrier decay rate. 
Note that the control barrier corridor $\barriercorridor_{\mathrm{full}}(\state)$ is a convex intersection of halfspaces and contains the safe robot state $\state \in \barriercorridor_{\mathrm{full}}(\state)$.
Hence, the metric projection (i.e., the closest point) onto the full control barrier corridor $\barriercorridor_{\mathrm{full}}(\state)$  can be computed efficiently using state-of-the-art off-the-shelf solvers \cite{boyd_vandenberghe_ConvexOptimization2004}.
Moreover, its convex shape depends on the control barrier functions and the ratio of the control gain $\kappa$  and the barrier rate $\alpha$.
An important question is: how do the control barrier functions influence the barrier corridor, and what are the potential consequences of (mis)matching the control convergence rate with the barrier decay rate?
We observe that a higher barrier decay rate increases system confidence, aggressiveness, and risk-taking, whereas a higher control convergence rate leads to caution, conservativeness, and risk avoidance (see \reffig{fig.full_control_barrier_corridor_dependency}).   
We show below that, in order to extend robot safety locally within the control barrier corridor, the convergence rate of the control and the allowed decay rate of barrier safety need to match (revealing a trade-off between safety and reactiveness) and the control barrier functions needs to be convex. 

\begin{figure}[t]
\centering
\begin{tabular}{@{}c@{\hspace{1mm}}c@{\hspace{1mm}}c@{}}
\raisebox{15mm}{\rotatebox{90}{\scalebox{0.5}{$p = 2$}} \hspace{-2mm}}
\includegraphics[width=0.32\columnwidth, trim=10mm 10mm 10mm 10mm, clip]{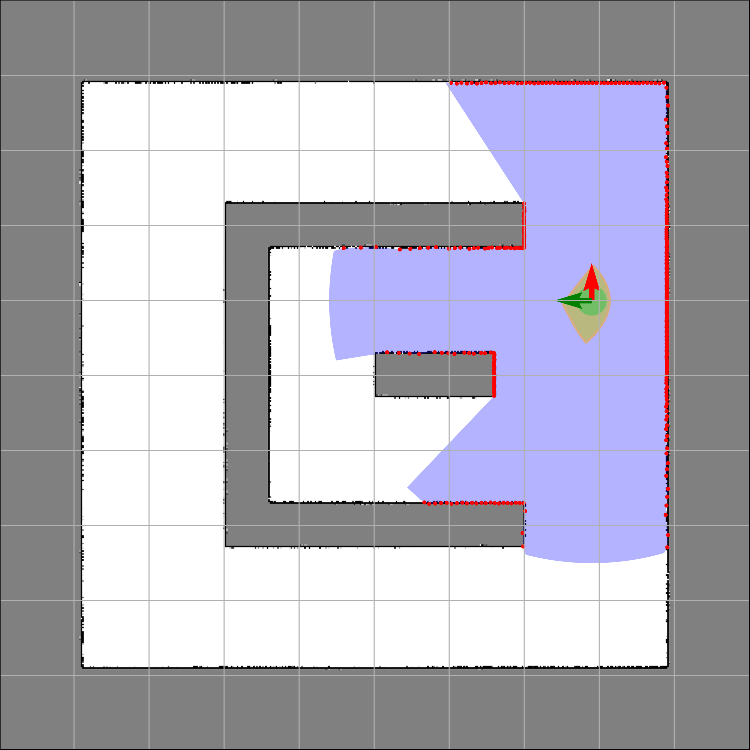}
&
\includegraphics[width=0.32\columnwidth, trim=10mm 10mm 10mm 10mm, clip]{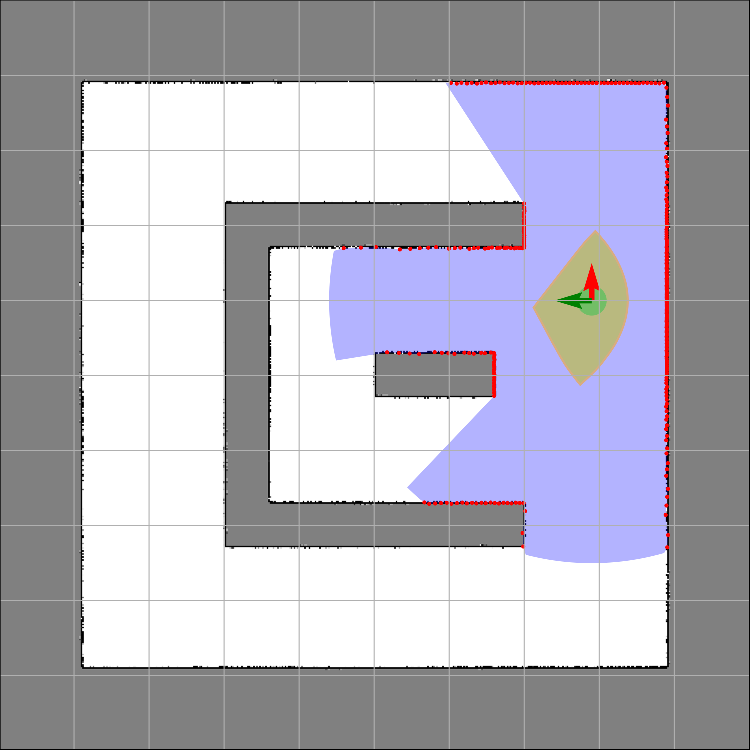}
&
\includegraphics[width=0.32\columnwidth, trim=10mm 10mm 10mm 10mm, clip]{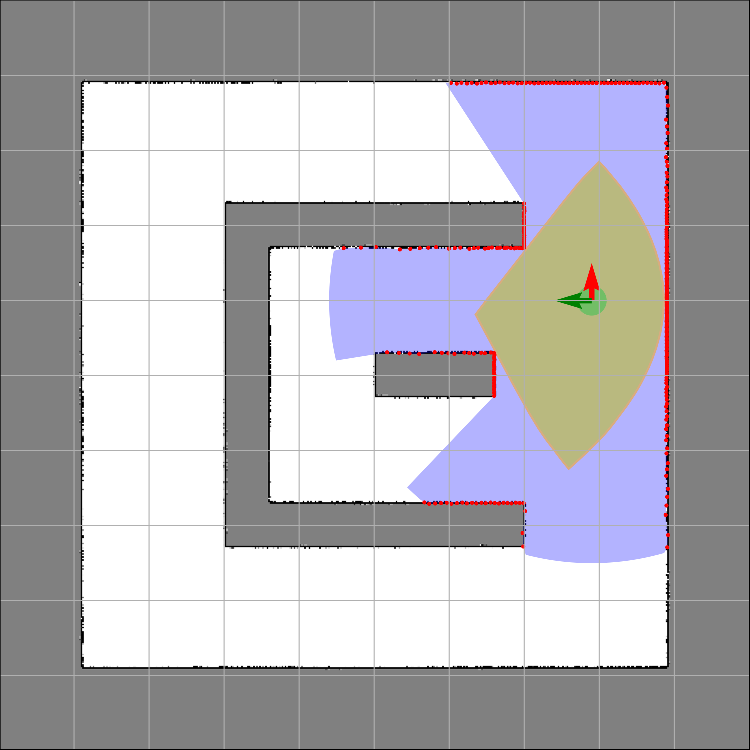}
\\
\raisebox{15mm}{\rotatebox{90}{\scalebox{0.5}{$p = 1$}} \hspace{-2mm}} 
\includegraphics[width=0.32\columnwidth, trim=10mm 10mm 10mm 10mm, clip]{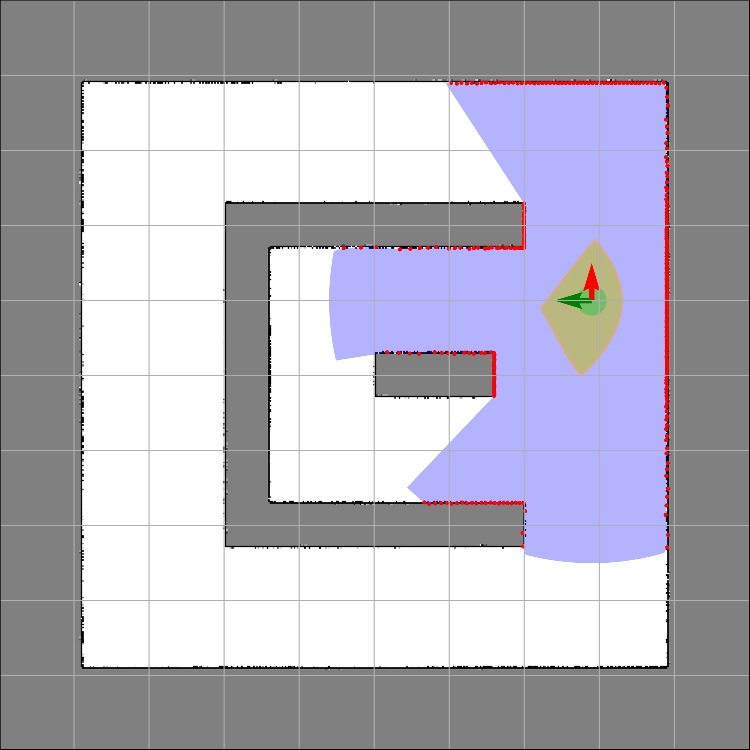}
& 
\includegraphics[width=0.32\columnwidth, trim=10mm 10mm 10mm 10mm, clip]{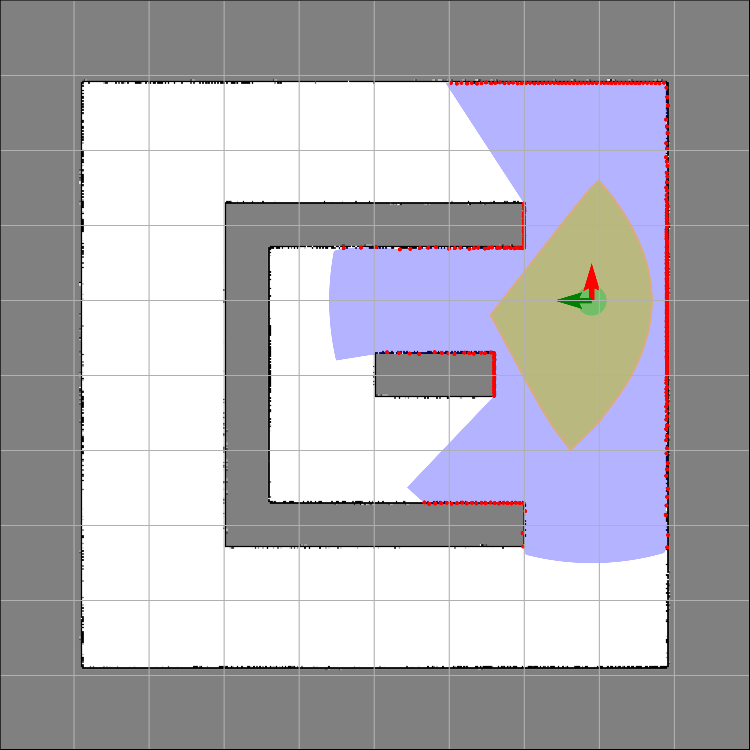}
& 
\includegraphics[width=0.32\columnwidth, trim=10mm 10mm 10mm 10mm, clip]{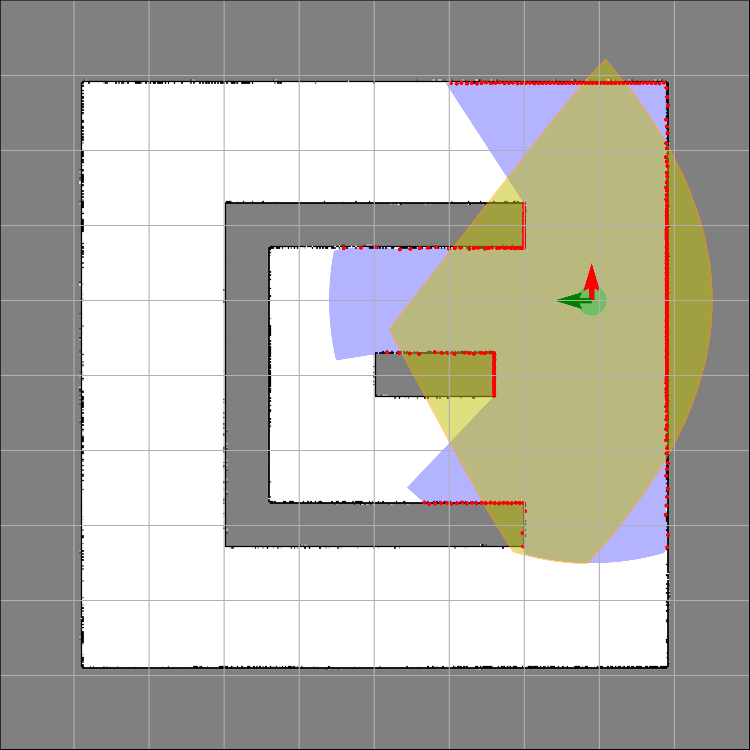}
\\
\raisebox{15mm}{\rotatebox{90}{\scalebox{0.5}{$p = 0.5$}} \hspace{-2mm}} 
\includegraphics[width=0.32\columnwidth, trim=10mm 10mm 10mm 10mm, clip]{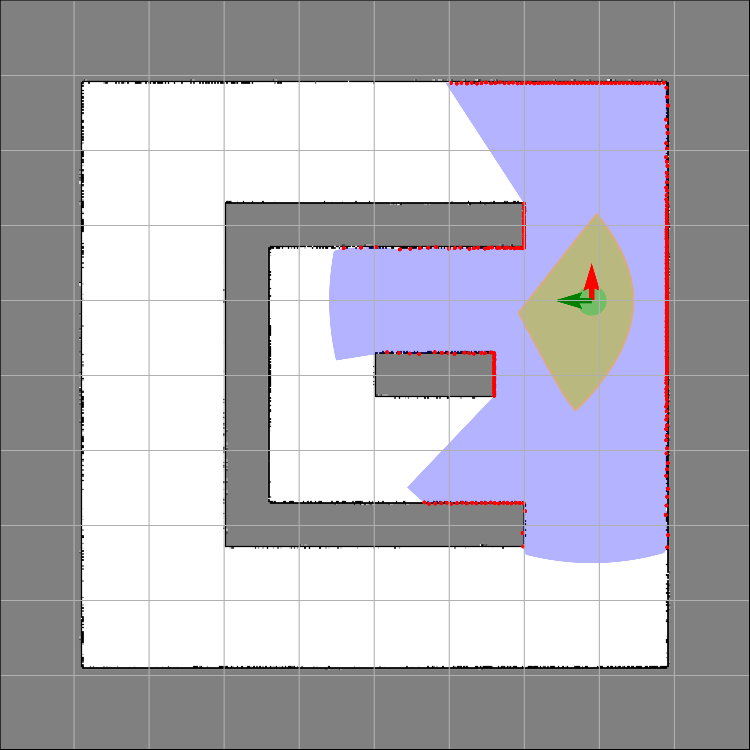}
&
\includegraphics[width=0.32\columnwidth, trim=10mm 10mm 10mm 10mm, clip]{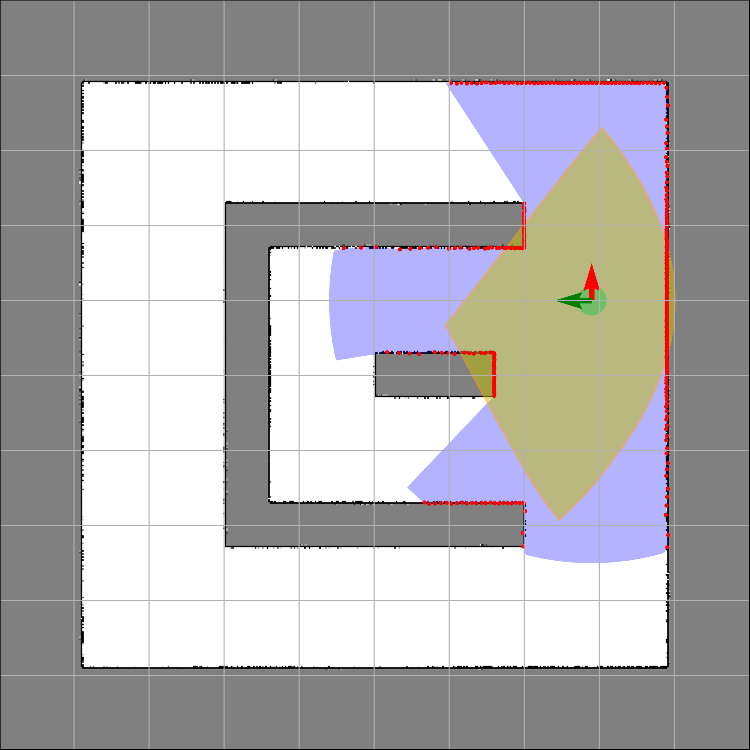}
& 
\includegraphics[width=0.32\columnwidth, trim=10mm 10mm 10mm 10mm, clip]{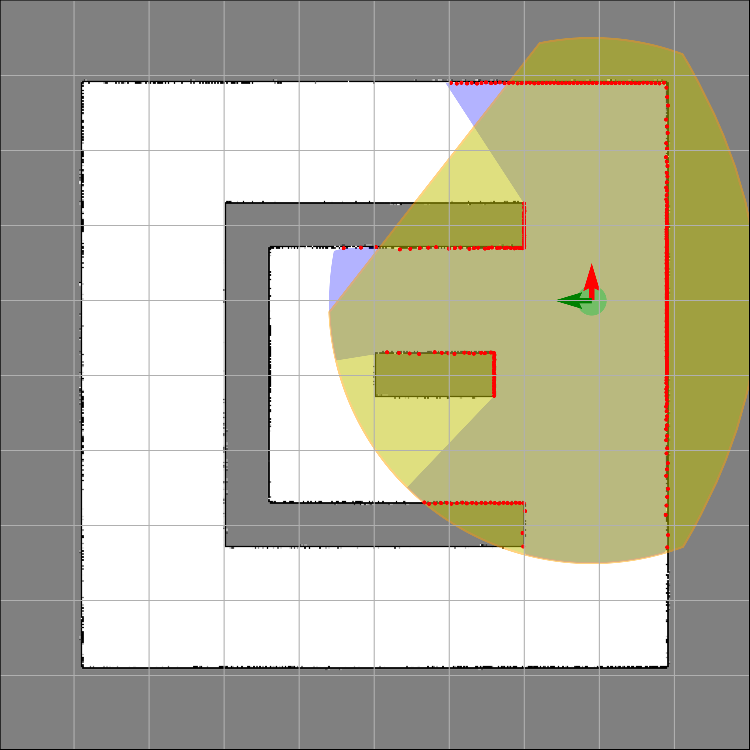}
\\[-2mm]
\scalebox{0.5}{$\alpha = 0.5 \kappa$} & \scalebox{0.5}{$\alpha =\kappa$} & \scalebox{0.5}{$\alpha = 2 \kappa$}   
\end{tabular}
\vspace{-3mm}
\caption{The influence of the proportional control gain $\ctrlgain$, barrier decay rate $\barrierrate$, and the order $p$ of the power distance on control barrier corridors (yellow) of a fully actuated robot (cyan) relative to sensed obstacle points (red).
A control barrier corridor corresponds to a local safe zone around the robot (cyan), free of obstacles (red), for a pair of matching barrier rate and control gain ($\barrierrate=\ctrlgain$), and convex barrier functions ($p \geq 1$).
The conservativeness of the barrier corridor increases with higher control gain ($\barrierrate < \ctrlgain$) and higher barrier convexity ($p > 1$), and it becomes less accurate with increasing barrier rate ($\barrierrate > \ctrlgain$) and increasing barrier concavity ($p < 1$).     
}
\label{fig.full_control_barrier_corridor_dependency}
\end{figure}

\begin{example}
\emph{(A Circular Mobile Robot around Obstacles)} Consider a fully actuated circular mobile robot with center position $\state \in \R^2$ and body radius $\radius>0$, moving around point obstacles at $\vect{q}_1, \ldots, \vect{q}_{\nbarrier} \in \R^2$, where the safety distance between the robot and the obstacles is measured using the $p^{\text{th}}$-order power-distance control barrier function, $h_{i}(\state) = \norm{\state - \vect{q}}^p  - \radius^p$. 
The associated control barrier corridor $\barriercorridor_{\mathrm{full}}(\state)$ around the robot position $\state$ is  given by 
\begin{align*}
\scalebox{0.85}{$\barriercorridor_{\mathrm{full}}(\state)\! =\! \bigcap\limits_{i=1}^{\nbarrier} \clist{\goal\!\Big | \tr{(\vect{x} \!-\! \vect{q}_i)\!}\!(\goal \! \!-\! \vect{q}_i) \!\geq\! \tfrac{p \ctrlgain - \barrierrate}{p\ctrlgain} \norm{\state \!-\! \vect{q}_i}^2 + \tfrac{\barrierrate}{p\ctrlgain}  \tfrac{\radius^2}{\norm{\state - \vect{q}_i}^{p-2}}}.$}
\end{align*}
As illustrated in \reffig{fig.full_control_barrier_corridor_dependency},  when the control gain and the barrier decay rate are equal (i.e., $\barrierrate = \ctrlgain$) and the true Euclidean distance to collision (i.e., $p=1$) is used, the control barrier corridor tightly separates the obstacle from the robot using supporting hyperplanes tangent to the configuration space obstacles. For larger values of the barrier decay rate (i.e., $\barrierrate > \ctrlgain$), the robot becomes aggressive, overconfident, and hallucinatory, as the barrier corridor fails to accurately separate the configuration space obstacle. Conversely, for larger values of the control gain (i.e., $\barrierrate < \ctrlgain$), the robot behaves conservatively, underconfident, and panicked, as the barrier corridor includes an extra margin from the configuration space obstacles.
Also note that increasing the order $p$ of the power distance increases the strength of convexity, which also results in increasing conservativeness.     
\end{example}

Hence, a critical question is: what determines the optimal ratio between the control gain $\kappa$ and the barrier decay rate $\alpha$, and how does this ratio relate to the convexity of the control barrier functions?
A potential answer is when the geometry of safety aligns with the dynamics of safety.
  
\begin{proposition}\label{prop.geometry_meets_dynamics}
\emph{(The Geometry of Safety Meets the Dynamics of Safety)} If the control barrier functions $h_{1}(\state), \ldots, h_{\nbarrier}(\state)$ are convex%
\footnote{Note that if the control barrier functions $h_{i}(\state)$ are strictly convex, i.e.,
\begin{align*}
h_i(\vect{y}) > h_i(\vect{x}) + \tr{\nabla h_i(\vect{x})} (\vect{y} - \vect{x}) \quad \forall \vect{y} \neq \vect{x},
\end{align*}
then the control barrier corridor $\barriercorridor_{\mathrm{full}}(\state)$ of a safe state $\state$ (with $h_i(\state) \geq 0$ for all $i$) of a fully actuated system consists of strictly safe goals, i.e.,
\begin{align*}
h_i(\goal) > 0 \quad \forall \goal \in  \barriercorridor_{\mathrm{full}}(\state)\setminus \clist{\state}.
\end{align*}
Similarly,  if the control barrier functions $h_i(\state)$ are strongly convex, i.e.,
\begin{align*}
h_i(\vect{y}) \geq h_i(\vect{x}) + \tr{\nabla h_i(\vect{x})} (\vect{y} - \vect{x}) + \frac{\mu}{2}\norm{\vect{y} - \vect{x}} \quad \forall \vect{y}, \vect{x}
\end{align*}
then the control barrier corridor $\barriercorridor_{\mathrm{full}}(\state)$ includes strongly safe goals, i.e., 
\begin{align*}
h_i(\goal) \geq  \frac{\mu}{2}\norm{\goal - \state}^2 \quad \forall \goal \in \barriercorridor_{\mathrm{full}}(\state),
\end{align*}
where $\mu > 0$ is the strength of convexity and the quadratic growth tendency.
}
and the proportional control gain $\ctrlgain$ and the barrier decay rate $\barrierrate$ are the same (i.e., $\barrierrate = \ctrlgain$), then any goal state $\goal \in \barriercorridor_{\mathrm{full}}(\vect{x})$ in the convex control barrier corridor
\begin{align*}
\barriercorridor_{\mathrm{full}}(\state)\!=\! \bigcap\limits_{i=1}^{\nbarrier}\clist{\goal \in \R^{\nstate} \Big | h_i(\state) +  \tr{\nabla h_i(\state)\!}\! (\goal \!- \!\state) \!\geq \! 0}
\end{align*}
of a safe state $\state$ (i.e., $h_i(\vect{x}) \geq 0$ for all $i$) is also safe, due to the first-order condition of convexity, i.e.,
\begin{align}
h_i(\vect{x}^*) \geq h_i(\vect{x}) + \tr{\nabla h_i(\vect{x})}(\vect{x}^* - \vect{x}) \geq 0,
\end{align} 
which makes $\barriercorridor_{\mathrm{full}}(\state)$ a local safe neighborhood around $\state$.
\end{proposition}

\begin{proof}
The result follows from the definition of convexity (in particular, the first-order condition of convexity) as follows: Any goal $\goal \in \barriercorridor_{\mathrm{full}}(\state)$ satisfies
\begin{align*}
- \ctrlgain \tr{\nabla h_i(\state)}(\state - \goal) \geq - \barrierrate h(\state)
\end{align*}
which can be rearranged for $\barrierrate = \ctrlgain$ into the first-order convexity condition of the control barrier function $h_i$ as
\begin{equation*}
h_i(\vect{x}^*) \geq h_i(\vect{x}) + \tr{\nabla h_i(\vect{x})}(\vect{x}^* - \vect{x}) \geq 0. \qedhere
\end{equation*}
\end{proof}

Another practical question is: Once a goal in the control barrier corridor is selected, how can we ensure that the goal remains safely reachable for all future times to ensure recursive feasibility and persistent goal control?
Accordingly, inspired by higher-order control barrier functions \cite{xiao_belta_TAC2022, tan_cortez_dimarogonas_TAC2022}, and using the constraint of the control barrier corridor, we define the goal-control barrier function $h_{i,\goal}(\state)$ as
\begin{align}\label{eq.goal_control_barrier_function}
h_{i, \goal}(\state) = - \ctrlgain \tr{\nabla h_i(\state)} (\state - \goal) + \barrierrate h_i(\state)
\end{align} 
which results in the following safe control constraint for goal safety for the fully actuated kinematic system $\dot{\vect{x}} = \vect{u}$:
\begin{align}
\dot{h}_{_i,\goal}(\state) &= \tr{\nabla h_{i,\goal} (\state)} \ctrl \geq -\barrierrate h_{i,\goal}(\state)
\end{align}
where the gradient of the goal-control barrier function $h_{i,\goal}(\state)$ is given in terms of  the Hessian $\nabla^2 h_i(\state)$  and gradient $\nabla h_i(\state)$ of the control barrier function $h_{i}(\state)$ as
\begin{align}
\nabla h_{i,\goal} (\state) = - \ctrlgain \nabla^2 h_{i}(\state) (\state - \goal) + (\barrierrate - \ctrlgain) \nabla h_{i}(\state).
\end{align}
The safety feasibility of goal-control barrier functions under proportional control ensures that the same goal can be persistently selected from a continuously evolving barrier~corridor.

\begin{proposition}\label{prop.persistent_safe_goal_control}
\emph{\!(Persistent and Safe Goal Control)} For any set of convex control barrier functions $h_{1}(\state), \ldots, h_{\nbarrier}(\state)$, and any pair of matching the control gain $\ctrlgain$ and the barrier rate $\barrierrate$ (i.e., $\barrierrate = \ctrlgain$), the fully actuated system $\dot{\state} = \ctrl$ under the proportional goal control \mbox{$\ctrl = -\ctrlgain(\state - \goal)$} asymptotically and safely brings any initial safe state $\state(0) \in \R^{\nstate}$ (with $h_i(\state(0)) \geq 0$ for all $i$) to any safe goal $\goal \in \barriercorridor_{\mathrm{full}}(\state(0))$, while ensuring that the goal $\goal$ remains in the barrier corridor $\barriercorridor_{\mathrm{full}}(\state(t))$ along the closed-loop trajectory $\state(t)$ for all future times $t \geq 0$, i.e.,
\begin{align*}
&h_i(\state(t)) \geq 0 \quad \forall t \geq 0, i = 1, \ldots, m, 
\\
&\goal \!\in \barriercorridor_{\mathrm{full}}(\state(t)) \quad \forall t \geq 0,
\\
&\!\!\lim_{t \rightarrow \infty} \state(t) = \goal.
\end{align*}
\end{proposition}
\begin{proof}
Starting at $t = 0$ from $\state(0) = \state_0$, the closed-loop trajectory of $\dot{\state} = -\ctrlgain (\pos - \goal)$ is given by 
\begin{align*}
\state(t) = e^{-\kappa t} \state_0 + (1 - e^{-\kappa t}) \goal, \quad t \geq 0,
\end{align*}
which converges exponentially to $\goal$ and corresponds to the safe line segment $[\state_{0}, \goal] \subseteq \barriercorridor_{\mathrm{full}}(\state_0)$  within the convex safe barrier corridor  $\barriercorridor_{\mathrm{full}}(\state_0)$, see \refprop{prop.geometry_meets_dynamics}.

Moreover, since $\goal \in \barriercorridor_{\mathrm{full}}(\state_0)$, we have $h_{i,\goal}(\state_0) \geq 0$ at the start at $t = 0$. 
Hence, the persistent inclusion of  the goal in the barrier corridor (i.e., $\goal \in \barriercorridor_{\mathrm{full}}(\state(t))$) follows from the fact that the safety feasibility of the goal-control barrier function is ensured 
under proportional control, i.e.,
\begin{align*}
\dot{h}_{i,\goal}(\state) 
  &= -\ctrlgain \tr{\nabla h_{i,\goal}(\state)} (\state - \goal) \\
  &= \ctrlgain^2 \tr{(\state - \goal)} \nabla^{2} h(\state) (\state - \goal) \\
  &\geq 0 \geq -\alpha h_{i,\goal}(\state),
\end{align*}
since $\nabla h_{i,\goal} (\state) = - \ctrlgain \nabla^2 h_{i}(\state) (\state - \goal)$ for $\barrierrate = \ctrlgain$, and the Hessian of a convex function is positive definite.
\end{proof}

The convexity of control barrier functions plays a critical role in persistent and safe goal control and local safety zone identification, which is difficult to achieve with single control barrier functions and is often unintentionally lost when composing multiple barrier functions into a single one.

\begin{remark}\label{rem.multiple_control_barrier_composition}
\emph{\mbox{(To Compose or Not to Compose)}}
Multiple control barrier functions, $h_1(\state), \ldots, h_{m}(\state)$, can be smoothly composed into one single barrier function \cite{molnar_ames_CSL2023}, for example, using their soft minimum or product as
\begin{align*}
\scalebox{0.96}{$h_{\smin} (\state) = -\tfrac{1}{\softness} \log \plist{\scalebox{0.9}{$\sum\limits_{i=1}^{\nbarrier} e^{-\softness h_i(\state)}$}\!\!}$}, 
\quad
\scalebox{0.96}{$h_{\mathrm{prod}}(\state) = \prod\limits_{i=1}^{\nbarrier} h_{i}(\state)$},
\end{align*}
whose gradients are given by 
\begin{align*}
\scalebox{0.82}{$\nabla h_{\mathrm{smin}}(\vect{x}) = \dfrac{\sum\limits_{i=1}^{m}e^{-\lambda h_i (\vect{x})} \nabla h_i (\vect{x})}{\sum\limits_{i=1}^{m}e^{-\lambda h_i (\vect{x})}}$},
\quad \!
\scalebox{0.82}{$\nabla h_{\mathrm{prod}}(\vect{x}) = h_{\mathrm{prod}}(\vect{x}) \! \sum\limits_{i=1}^{\nbarrier} \! \dfrac{\nabla h_i(\state)}{h_i(\state)}$}, 
\end{align*}
where $\softness> 0$ is a fixed smoothing parameter.
Although composing multiple control barrier functions into a single barrier function is preferred for explicitly determining safe control \cite{molnar_ames_CSL2023}, such functional compositions often lead to a loss of convexity, which limits the reliable extension of robot state safety to a local safe zone around it by a single safety constraint, as illustrated in \reffig{fig.softmin_control_barrier_corridor}, and introduce other known technical issues (e.g., the soft minimum can cause conservatism since it is a lower approximation of the true minimum with error, and gradient scaling and vanishing in barrier products can lead to numerical issues and conservatism).
\end{remark}

\begin{figure}
\begin{tabular}{@{\hspace{0.5mm}}c@{\hspace{1mm}}c@{\hspace{1mm}}c@{}}
\includegraphics[width=0.32\columnwidth, trim=10mm 10mm 10mm 10mm, clip]{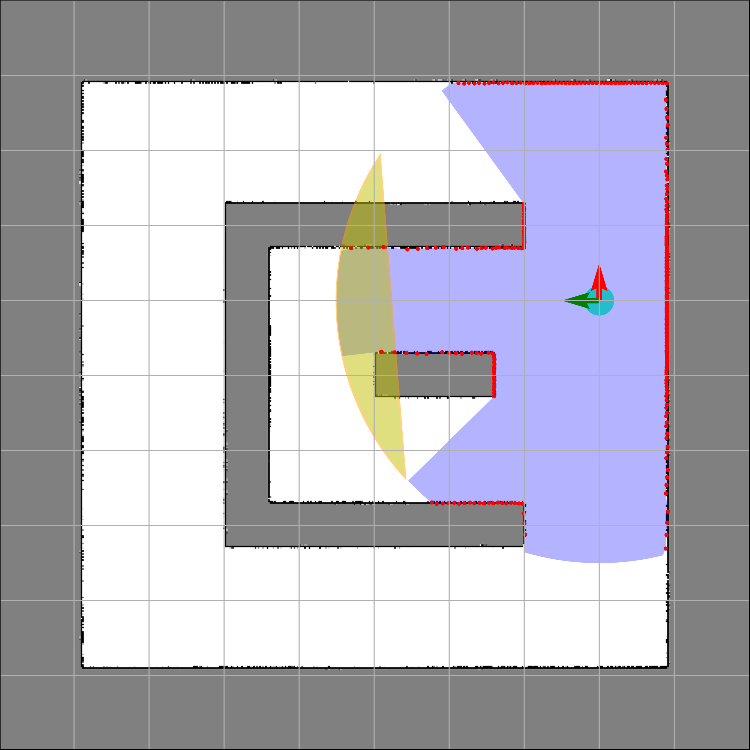}
&
\includegraphics[width=0.32\columnwidth, trim=10mm 10mm 10mm 10mm, clip]{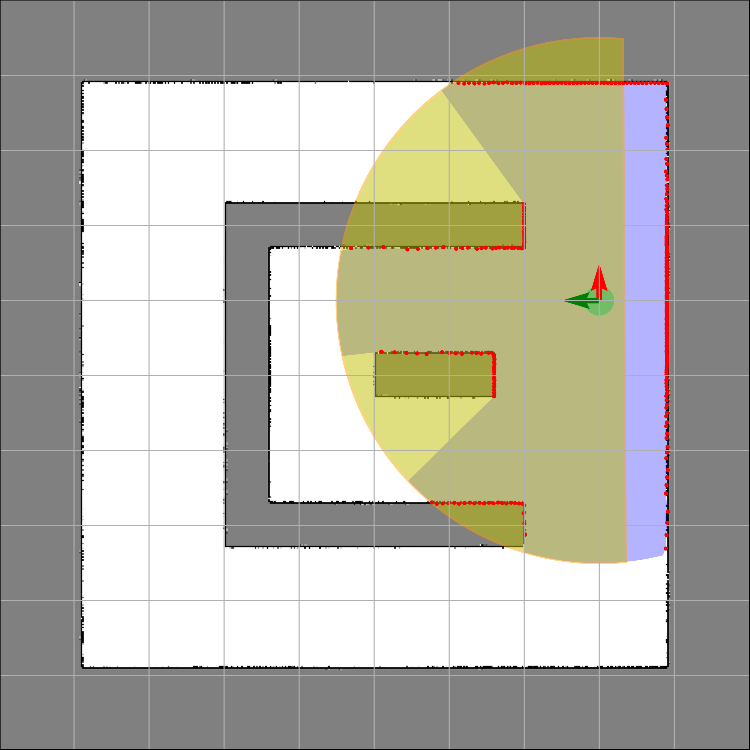}
&
\includegraphics[width=0.32\columnwidth, trim=10mm 10mm 10mm 10mm, clip]{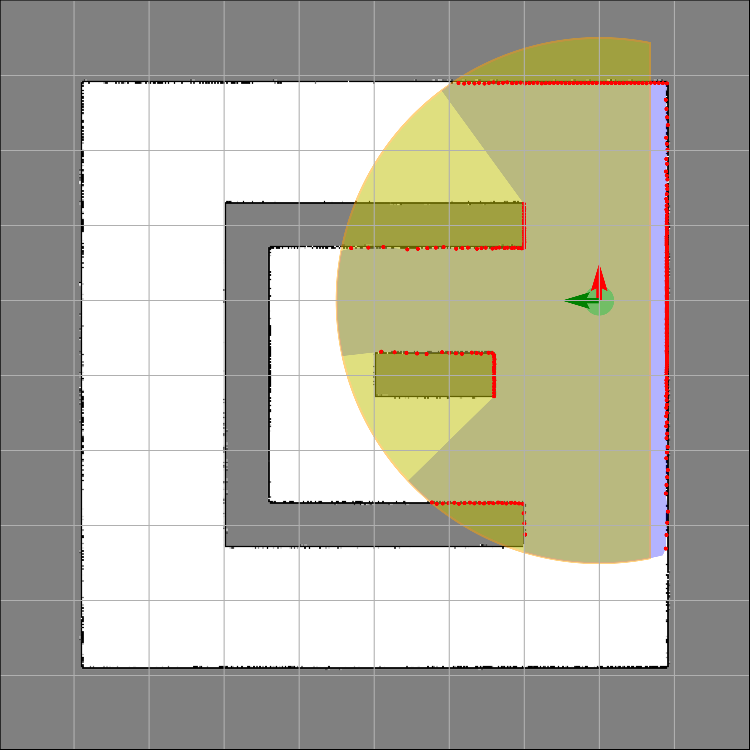} 
\end{tabular}
\vspace{-2mm}
\caption{Control barrier corridor (yellow) of the soft minimum of multiple barrier functions based on the collision distance of a circular robot body (cyan) to sensed obstacle points (red), where the proportional gain and linear barrier decay rates are the same, i.e., $\barrierrate=\ctrlgain$. 
Due to the lack of convexity and approximation error, the soft-min barrier composition is unable to identify a collision-free convex neighborhood around the robot as a control barrier corridor.
(left) $\softness = 2$: small softness results in large approximation error and inaccurate safety assessment of the robot position. 
(middle) $\softness = 10$: sufficiently large softness ensures robot safety but still inaccurately approximates the minimum distance to collision. 
(right) $\softness = 100$: very large softness accurately approximates the minimum distance to collision, but its gradient becomes less smooth.
}
\label{fig.softmin_control_barrier_corridor}
\vspace{-3mm}
\end{figure}

\subsection{Control Barrier Corridors of Kinematic Unicycle Systems}

Control-affine systems of the form $\dot{\vect{x}} = f(\vect{x}) + g(\vect{x}) \vect{u}$ allow for a convex quadratic optimization formulation of safety filtering of any given desired reference control as in \refeq{eq.safety_filtering}. 
However, the general nonlinear transformation of the control input (i.e., $g(\vect{x}) \vect{u}$) makes it difficult to fully understand, in geometric terms, the local influence of control on state safety, for example via control barrier corridors.
In this part, as an example of control-affine systems, we present a construction of control barrier corridors for kinematic unicycle systems under a standard geometric feedback control policy for goal-point navigation. 
In particular, we consider a kinematic unicycle robot, with position $\pos \in \R^2$ and forward orientation \mbox{$\orient \in [-\pi, \pi)$}, whose equations of motion are given by
\begin{align} \label{eq.unicycle_dynamics}
\dot{\pos} = \linvel \ovect{\orient}, \quad \text{ and } \quad  \dot{\orient} = \angvel,
\end{align}  
where $\linvel \in \R$ and $\angvel \in \R$ are the scalar control inputs specifying the linear and angular velocities of the unicycle, respectively. 
Note that the unicycle dynamics are control-affine but underactuated, subject to the nonholonomic motion constraint of no lateral (sideways) motion, i.e., \mbox{$\scalebox{0.9}{$\tr{\nvect{\orient}}$} \dot{\pos} = 0$}.  

Using a standard inner-outer loop control architecture\reffn{fn.unicycle_goal_ctrl_angular_linearization} \cite{astolfi_JDSMC1999, arslan_koditschek_IJRR2019}, the unicycle dynamics can be globally asymptotically brought to any goal position $\goalpos \in \R^2$ by minimizing the distance to the goal (outer loop) and the angular alignment error (inner loop), with the control inputs given by\reffn{fn.optimal_linear_velocity_unconstrained} 
\begin{subequations} \label{eq._unicycle_control}
\begin{align}
\linvel &= -\lingain \tr{\ovectsmall{\orient}} (\pos - \goalpos), 
\label{eq.unicycle_linear_control}
\\
\angvel &= \anggain \mathrm{atan2}\plist{\tr{\nvectsmall{\orient}\!}\!(\goalpos \!- \pos), \tr{\ovectsmall{\orient}\!}\!(\goalpos\! - \pos)\!} 
\label{eq.unicycle_angular_control}
\end{align} 
\end{subequations}
where $\mathrm{atan2}(y,x)$ is the two-argument arctangent function that returns the polar angle of the point $(x,y)$ in $[-\pi, \pi)$.  

\addtocounter{footnote}{1}
\footnotetext{\label{fn.unicycle_goal_ctrl_angular_linearization}
Alternatively, using angular feedback linearization \cite{tarshahani_isleyen_arslan_ECC2024}, one may also globally bring the unicycle dynamics in \refeq{eq.unicycle_dynamics} to any goal position $\goalpos \in \R^{2}$ using the following linear and angular velocity control inputs    
\begin{align*}
\linvel &= -\lingain \tr{\ovectsmall{\orient}} (\pos - \goalpos)
\\
\angvel &= \ctrlgain_{\omega} \varepsilon_{\goal}(\pos, \orient) + \linvel \frac{\sin \varepsilon_{\goal}(\pos, \orient)}{\norm{\pos - \goalpos}}
\end{align*}
where the angular alignment error is defined as  
\begin{align*}
\varepsilon_{\goal}(\pos, \orient) = \mathrm{atan2}\plist{\tr{\nvectsmall{\orient}}(\goalpos- \pos), \tr{\ovectsmall{\orient}}(\goalpos- \pos)}
\end{align*}
and it evolves under the unicycle dynamics in \refeq{eq.unicycle_dynamics} as
\begin{align*}
\frac{\diff}{\diff t}\varepsilon_{\goal}(\pos, \orient) = - \angvel + \linvel \frac{\sin \varepsilon_{\goal}(\pos, \orient)}{\norm{\pos - \goalpos}}.
\end{align*} 
Note that the same results and discussions on unicycle control barrier corridors and persistent unicycle goal control still hold. 
}

\addtocounter{footnote}{1}
\footnotetext{\label{fn.optimal_linear_velocity_unconstrained}
Note that one may consider the fully actuated proportional goal control $\ctrl_{d}= -\lingain(\pos - \goalpos)$ as a desired reference control and observe that the unicycle linear velocity control in \refeq{eq.unicycle_linear_control} is the optimal velocity input, being the closest to the reference control and equivalent to unconstrained safety filtering of the reference control, i.e.,
\begin{align*}
\linvel^* = -\lingain \tr{\ovectsmall{\orient}}(\pos - \goalpos) = \argmin_{\linvel \in \R} \left \| \ovectsmall{\orient}\linvel + \lingain(\pos - \goalpos) \right \|^2.
\end{align*}
}

A practical feature of circular unicycle robots (e.g., cleaning robots) is that they can freely perform turns in place with no collision with obstacles.
Accordingly, for a given set of control barrier functions $h_1(\pos), \ldots, h_{\nbarrier}(\pos)$ assessing the safety of the unicycle position $\pos$ irrespective of its orientation $\orient$ (e.g., power distance to collision), we determine the control barrier corridor of the closed-loop unicycle system as the set of goal positions that yields safe control as 
{\small
\begin{align} \label{eq.unicycle_control_barrier_corridor}
\!\!\barriercorridor_{\mathrm{uni}}(\pos, \orient) \!&=\!\! \scalebox{0.87}{$\bigcap\limits_{i=1}^{\nbarrier}\!\clist{\!\goalpos  \!\! \in \!\R^2 \Big | \!-\! \lingain \tr{\nabla h_i(\pos)\!} \!\ovectsmall{\orient} \!\tr{\ovectsmall{\orient}\!}\!\!(\pos \!-\! \goal) \! \geq \! -\barrierrate h_i (\pos) \!}$} \nonumber
\\
& \hspace{-16mm} = \!\!\scalebox{0.85}{$\bigcap\limits_{i=1}^{\nbarrier}\!\clist{\!\goalpos  \!\! \in \!\R^2 \Big | \!-\! \lingain \tr{\nabla h_i(\pos)\!} \! \plist{ \!\mat{I} \!-\! \nvectsmall{\orient} \!\tr{\nvectsmall{\orient}\!}\!}\!(\pos \!-\! \goal) \! \geq \! -\barrierrate h_i (\pos) \!}$} \!\!\!
\end{align}   
}%
where $\alpha > 0$ is a constant barrier decay rate.
Note that the unicycle control barrier corridor  $\barriercorridor_{\mathrm{uni}}(\pos,\orient)$ in \refeq{eq.unicycle_control_barrier_corridor} 
is larger than, and contains, the fully actuated control barrier corridor  $\barriercorridor_{\mathrm{full}}(\pos)$ in \refeq{eq.full_control_barrier_corridor} (for the same control gains $\lingain = \ctrlgain$), i.e., 
$\barriercorridor_{\mathrm{full}}(\pos) \subset \barriercorridor_{\mathrm{uni}}(\pos, \orient)$,  but it does not define a local safe neighborhood of a unicycle pose,  because the nonholonomic motion constraint inherently ensures lateral safety in the sideways direction. 
On the other hand, ensuring safety in all directions results in the full control barrier corridor, i.e.,
\begin{align}
\barriercorridor_{\mathrm{full}}(\pos) = \!\!\! \bigcap_{-\pi \leq \orient \leq \pi} \!\! \barriercorridor_{\mathrm{uni}}(\pos, \orient),
\end{align}   
which makes the full control barrier corridor a local safe zone candidate for goal selection in unicycle control.

Another known technical challenge of unicycle dynamics is that the nonholonomic motion constraint does not permit continuous smooth control toward a goal position $\goalpos$ at the critical barrier safety boundary, where $h_i(\goalpos) = 0$ for some $i$, and therefore requires an open safe neighborhood around the goal location to reach it with continuous and smooth control \cite{arslan_koditschek_IJRR2019, latha_arslan_ICRA2025}.
Hence, given any desired extra safety margin $\varepsilon \geq 0$, we define an $\varepsilon$-safer full control barrier corridor of a safe state $\state$, with $h_i(\state) \geq 0$ for all $i$, as
{
\begin{align}\label{eq.safer_full_control_barrier_corridor}
\!\!\barriercorridor_{\mathrm{full}, \varepsilon}(\state)&\!=\! \scalebox{0.95}{$ \bigcap\limits_{i=1}^{\nbarrier}\clist{\goal \!\!\in\! \R^{2} \Big |\! -\!\kappa \tr{\nabla h_{i,\varepsilon} (\state)\!}\!(\state \!-\! \goal) \!\geq\! - \barrierrate h_{i,\varepsilon}(\state)\!}$} \nonumber
\\
& \hspace{-13mm} =\! \scalebox{0.95}{$\bigcap\limits_{i=1}^{\nbarrier} \clist{\goal \in \R^2 \Big |\! - \!\kappa \tr{\nabla h_{i} (\state)\!}\!(\state \!-\! \goal) \!\geq\! - \barrierrate (h_{i}(\state) \!-\! \varepsilon)\!}$} \!\!\!
\end{align}
}%
based on the $\varepsilon$-safer control barrier functions defined to be
\begin{align}\label{eq.safer_control_barrier_functions}
h_{i,\varepsilon}(\state) = h_{i}(\state) - \varepsilon,
\end{align}
where $\nabla h_{i,\varepsilon}(\state) = \nabla h_{i}(\state)$.
For persistent selection  of an $\varepsilon$-safer goal $\goal$ in $\barriercorridor_{\mathrm{full}, \varepsilon}(\state)$, we also define the $\varepsilon$-safer goal-control barrier functions as 
\begin{align}
h_{i, \goal, \varepsilon}(\state) &= - \kappa \tr{\nabla h_{i,\varepsilon} (\state)}(\state - \goal) + \barrierrate h_{i,\varepsilon}(\state)
\\
& = -\kappa \tr{\nabla h_{i} (\state)}(\state - \goal) + \barrierrate (h_{i}(\state) \!-\! \varepsilon)
\end{align} 
whose gradient is given by
\begin{align}
\nabla h_{i, \goal, \varepsilon}(\state) 
&= - \kappa \nabla^2 h_{i}(\state)(\state -\goal) +(\alpha - \kappa) \nabla h_{i}(\state). \!\!\!
\end{align}
Accordingly,  using control barrier functions $h_i(\state)$ and the $\varepsilon$-safer goal-control barrier functions $h_{i, \goal, \varepsilon}(\state)$,  we construct a safety filter for the reference linear velocity in \refeq{eq.unicycle_linear_control} of the unicycle dynamics in \refeq{eq.unicycle_dynamics} to determine the safe linear velocity input that ensures both the safety of the robot and the persistent selection of an $\varepsilon$-safe goal $\goalpos \in \barriercorridor_{\mathrm{full},\varepsilon}(\state)$~as\reffn{fn.unicycle_velocity_objective}%
{\small
\begin{align}\label{eq.safe_unicycle_linear_control}
\begin{array}{rl}
\text{minimize} & \plist{\linvel + \ctrlgain \tr{\ovectsmall{\orient}}(\pos-\goalpos)}^2
\\
\text{subject to} & \tr{\nabla h_i(\pos)}  \ovectsmall{\orient} \linvel \geq - \barrierrate h_i (\pos) \quad \forall i 
\\
& \tr{\nabla h_{i,\goal, \varepsilon}(\pos)} \ovectsmall{\orient} \linvel \geq -\barrierrate h_{i, \goal, \varepsilon} (\pos) \quad \forall i 
\end{array}
\end{align}
}%
and  we determine the angular velocity input as in \refeq{eq.unicycle_angular_control}.
Note that the safe unicycle velocity control optimization in \refeq{eq.safe_unicycle_linear_control} is always feasible for any safe robot position $\pos$ and $\varepsilon$-safer goal $\goal$ in $\barriercorridor_{\mathrm{full},\varepsilon}(\state)$ (since $\linvel = 0$ is always safe for both the robot and persistent goal selection), and can be explicitly solved, because it is a one-dimensional convex quadratic optimization problem.\reffn{fn.explicit_scalar_quadratic_optimization}

\addtocounter{footnote}{1}
\footnotetext{\label{fn.unicycle_velocity_objective}%
Note that the objective of the unicycle safety filter in \refeq{eq.unicycle_linear_control} is proportional to the squared Euclidean distance between the unicycle velocity $\ovectsmall{\orient}\linvel$ and the desired fully actuated proportional control  $-\ctrlgain(\pos - \goal)$ as 
\begin{align*}
\left\|-\ctrlgain(\pos - \goalpos) - \ovectsmall{\orient} \linvel\right \|^2 &=  \norm{\ctrlgain(\pos - \goalpos)}^2 + \linvel ^2 + 2 \ctrlgain \tr{\ovectsmall{\orient}}(\pos-\goalpos)
\\
& \hspace{-15mm} = \plist{\linvel + \ctrlgain \tr{\ovectsmall{\orient}}(\pos-\goalpos)}^2    + \plist{\ctrlgain \tr{\nvectsmall{\orient}}(\pos-\goalpos)}^2.
\end{align*}
}

\addtocounter{footnote}{1}
\footnotetext{\label{fn.explicit_scalar_quadratic_optimization}%
Consider a scalar-variable quadratic optimization problem  of the form:
\begin{align*}
\begin{array}{rl}
\underset{x \in \R}{\mathrm{minimize}} & (x - x_d)^2
\\
\mathrm{subject to} & a_i x \leq b_i \quad \forall i = 1, \ldots, m 
\end{array}
\end{align*}
where the constraint set corresponds to the closed interval  $[x_{\min}, x_{\max}] $, 
\begin{align*}
x_{\min} = \max_{\substack{i=1, \ldots, m \\ a_i \leq 0}} \tfrac{b_i}{a_i}, 
\quad
x_{\max} = \min_{\substack{i=1, \ldots, m \\ a_i \geq 0}} \tfrac{b_i}{a_i}.  
\end{align*} 
Therefore, the optimal solution $x^*$ exists if $x_{\min} \leq x_{\max}$, and it is given by $x_d$ if $ x_d \in [x_{\min}, x_{\max}]$; otherwise the by the closest bound, i.e.,
\begin{align*}
x^* = \left \{ \begin{array}{cl}
x_d & \text{if } x_{\min} \leq x_d \leq x_{\max}],
\\
x_{\min} & \text{if } x_d < x_{\min}\leq x_{\max},
\\
x_{\max} & \text{if } x_d > x_{\max} \geq x_{\min},
\\
\text{infeasible} & \text{otherwise}.
\end{array}
\right.
\end{align*}
}

\begin{proposition}\label{prop.persistent_safe_unicycle_navigation}
\emph{(Persistent and Safe Unicycle Goal Control)}
Given any set of convex control barrier functions $h_{i}(\state)$, a matching pair of control gain $\ctrlgain$ and barrier rate $\barrierrate$ (i.e., $\barrierrate = \ctrlgain$), and any positive safety margin $\varepsilon > 0$, the safe unicycle control in \refeq{eq.safe_unicycle_linear_control} and \refeq{eq.unicycle_angular_control} asymptotically and safely brings any safe unicycle pose $(\pos, \orient)$ (with $h_i(\pos) \geq 0$ for all $i$) to any $\varepsilon$-safer goal $\goalpos \in \barriercorridor_{\mathrm{full},\varepsilon}(\pos)$, while ensuring that the goal $\goalpos$ remains in the $\varepsilon$-safer barrier corridor $\barriercorridor_{\mathrm{full},\varepsilon}(\pos(t))$ along the closed-loop unicycle motion trajectory $(\state(t), \orient(t))$ for all $t \geq 0$, i.e.,
\begin{align*}
&h_i(\pos(t)) \geq 0 \quad \forall t \geq 0, i = 1, \ldots, m 
\\
&\goalpos \in \barriercorridor_{\varepsilon}(\pos(t)) \quad \forall t \geq 0,
\\
&\lim_{t \rightarrow \infty} \pos(t) = \goalpos.
\end{align*}
\end{proposition}
\begin{proof}
If $\barriercorridor_{\mathrm{full}, \varepsilon}(\state) = \varnothing$, then the statement holds trivially. Otherwise, the result can be verified as follows.

For convex barrier functions $h_{i}(\state)$ and a matching control gain and barrier rate $\barrierrate = \ctrlgain$, the full control barrier corridor $\barriercorridor_{\mathrm{full}}(\state)$ in \refeq{eq.full_control_barrier_corridor} is a local safe convex neighborhood of a safe unicycle position $\state$ (see \refprop{prop.geometry_meets_dynamics}).  
Moreover, by definition \refeq{eq.safer_full_control_barrier_corridor}, the $\varepsilon$-safer full control barrier corridor $\barriercorridor_{\mathrm{full}, \varepsilon}(\state)$ is a convex subset of $\barriercorridor_{\mathrm{full}}(\state)$.
Therefore, the straight line between the robot position $\pos$ and the goal position $\goal$ is safe and contained in  $\barriercorridor_{\mathrm{full}}(\state)$, i.e.,
\begin{align}
\goalpos \in \barriercorridor_{\mathrm{full}, \varepsilon}(\state) \subset \barriercorridor_{\mathrm{full}}(\state), 
\\
[\goalpos, \pos]\subseteq \barriercorridor_{\mathrm{full}}(\state).
\end{align}

The optimal safe velocity control always exists (since $\linvel = 0$ is feasible) and, by design, it ensures the safety of the robot position $\pos$ and the persistent selection of the goal $\goal$.
Moreover, it satisfies $- \ctrlgain \tr{\ovectsmall{\orient}}(\pos - \goal) \linvel \geq 0$ and so the distance to the goal decreases over time as 
\begin{align}
\frac{\diff}{\diff t} \norm{\pos - \goalpos}^2  = \tr{\ovectsmall{\orient}}(\pos - \goal) \linvel \leq 0 
\end{align}
which can be zero only when $\linvel = 0$.
If $\linvel = 0$, the angular velocity control in \refeq{eq.unicycle_angular_control} exponentially aligns the robot’s orientation with the goal $\goal$.  
Since the goal remains $\varepsilon$-safer in $\barriercorridor_{\mathrm{full}, \varepsilon}(\pos) \subset \barriercorridor_{\mathrm{full}}(\pos)$ for all times, the linear velocity $\linvel$ is guaranteed to become positive after a finite time and so the robot continuously move toward the goal while strictly decreasing the distance to the goal. 
\end{proof}

\subsection{Control Barrier Corridors of Linear Control Systems}

As a third example of control barrier corridors, we consider a linear time-invariant system: 
\begin{equation}\label{eq.linear_control_system}
\begin{aligned}
\dot{\state} &= A \state + B \ctrl, \\
\vect{y} &= C \state,
\end{aligned}
\end{equation}
where $\state \in \R^{\nstate}$ is the state (e.g., position and velocity), $\ctrl \in \R^{\nctrl}$ is the control input (e.g., acceleration), and $\youtput \in \R^{\noutput}$ is the output (e.g., position). The matrices \mbox{$A \in \R^{\nstate \times \nstate}$}, \mbox{$B \in \R^{\nstate \times \nctrl}$}, and \mbox{$C \in \R^{\noutput \times \nstate}$} are, respectively, the state, input, and output matrices of appropriate dimensions.

We assume that the pair $(A,B)$ is stabilizable and that $K \in \R^{\nctrl \times \nstate}$ is a feedback matrix such that $A + BK$ is Hurwitz.
We also assume that  $\begin{bmatrix} A & B \\ C  & 0\end{bmatrix}$ is invertible,\footnote{For second-, or any higher-order linear systems, the matrix $\begin{bmatrix} A & B \\ C & 0 \end{bmatrix}$ is invertible, for example, for the second-order $\ddot{\state} = \ctrl$ and $y = \state$, we have $A = \begin{bmatrix} 0 & 1 \\ 0 & 0\end{bmatrix}$, $B = \begin{bmatrix}0 \\ 1\end{bmatrix}$, and $C = \begin{bmatrix}
1 & 0 \end{bmatrix}$.} so that the output can be related to the state and control via the output-to-state and output-to-control matrices, $X \in \R^{\nstate \times \noutput}$ and $U \in \R^{\nctrl \times \noutput}$, given by
\begin{equation}
\begin{bmatrix}
X \\
U
\end{bmatrix} = 
\begin{bmatrix} 
A & B \\ C & 0
\end{bmatrix}^{-1} 
\begin{bmatrix}
0 \\
I
\end{bmatrix}.
\end{equation}
This ensures that the output-to-state-to-output mapping is the identity, i.e., $C X = I$ and $\youtput = C X \youtput$ for all $\youtput$, and that output regulation vanishes at steady state, i.e., $A X + B U = 0$ and $A X \youtput + B U \youtput = 0$ for all $\youtput$.
Hence, given any desired output $\youtput^*$, the state-feedback output regulation control policy \cite{output_regulation_huang2004nonlinear}:
\begin{equation} \label{eq.linear_output_regulation_control}
\begin{aligned}
u = k_{\youtput^*}(\state) &= K (\state - X \youtput^*) + U \youtput^* \\
&= K \state + (U - K X) \youtput^*,
\end{aligned}
\end{equation}
renders $\state^* = X \youtput^*$ an exponentially stable equilibrium of the closed-loop system dynamics:
\begin{equation} \label{eq:closed_loop_lor}
\begin{aligned} 
\dot{\state} &= (A + B K)(\state - X \youtput^*), \\
\youtput &= C \state,
\end{aligned}
\end{equation}
and ensures that the output $\youtput$ converges exponentially to $\youtput^*$.

Given a set of control barrier functions $h_{1}(\state), \ldots, h_{\nbarrier}(\state)$, the control barrier corridor of the linear control system in \refeq{eq.linear_control_system}, under the output regulation control in \refeq{eq.linear_output_regulation_control}, contains all desired system outputs $\youtput^*$ that ensure safety feasibility at a safe system state $\state$ with $h_i(\state) \geq 0$ as
{\small
\begin{align}
\barriercorridor_{\mathrm{lor}}(\state) \!=\! \bigcap_{i=1}^{m}\clist{\youtput^* \big | \tr{\nabla h_{i}(\state)\!}\! (A \!+\! B K)(\state \!-\! X \youtput^*) \!\geq\! - \alpha (h_i(\state)\!)\!}
\end{align}  
}%
which is convex, being the intersection of halfspaces.
Note that, in general, the control barrier corridor $\barriercorridor_{\mathrm{lor}}(\state)$ of the linear output regulator system can be empty, especially for high-order systems, and that the current system output $\youtput = C \state$ is not necessarily in the barrier corridor $\barriercorridor_{\mathrm{lor}}(\state)$.
Hence, persistent desired output selection and regulation (as in \refprop{prop.persistent_safe_goal_control}) is not guaranteed in general.
A promising open research question is how to design a control policy and determine requirements on the control barrier functions for linear output regulation systems such that:
\begin{itemize}
\item (Recursive Feasibility) the control barrier corridor is nonempty, i.e., $\barriercorridor(\state) \neq \varnothing$ for any safe state with $h_{i}(\state) \geq 0$;

\item (Local Neighborhood) the control barrier corridor contains the current output, i.e., $\youtput = C \state \in \barriercorridor(\state)$ for any safe state with $h_{i}(\state) \geq 0$;

\item (Persistent Output Regulation) the control barrier corridor ensures persistent and safe output regulation, i.e., $h_i(\state(t)) \geq 0$ and $\youtput^* \in \barriercorridor(\state(t))$ for all $t \geq 0$, given that $h_i(\state(0)) \geq 0$ and $\youtput^* \in \barriercorridor(\state(0))$.
\end{itemize}
For example, if both $B$ and $C$ are invertible, all these desired features can be satisfied via full state-feedback linearization by embedding the fully-actuated reference system dynamics $\dot{\state} = -\kappa (\state - \state^*)$, where $\state^* = C^{-1}y^*$, by setting $X = C^{-1}$, $U = - B^{-1}A C^{-1}$, and $K = -B^{-1}(\kappa I + A)$.

In the general case, similar to the goal-control barrier functions in \refeq{eq.goal_control_barrier_function}, to ensure the inclusion of the current output in the control barrier corridor and persistent desired output regulation, one may consider the current-output and goal-output control barrier functions, respectively, defined as
\begin{align}
h_{i,\youtput}(\state) = \tr{\nabla h_{i}(\state)\!} (A \!+\! B K)(I \!-\!XC)\state + \alpha (h_i(\state)\!), \!\!\!
\\
h_{i,\youtput^*}(\state) = \tr{\nabla h_{i}(\state)\!} (A \!+\! B K)(\state\!-\!X\youtput^*) + \alpha (h_i(\state)\!), \!\!\!
\end{align}
where $h_{i,\youtput}(\state) \geq 0$ implies $\youtput = C \state \in \barriercorridor_{\mathrm{lor}}(\state)$,  and  $h_{i,\youtput^*}(\state) \geq 0$ implies \mbox{$\youtput^* \in \barriercorridor_{\mathrm{lor}}(\state)$}.
Thus, the safety feasibility of output-control barrier functions can be used as constraints for recursive feasibility and persistent regulation.

Finally, to provide a sufficient condition for persistent output regulation, we consider the following trust region:
\begin{equation}\label{eq:trust_region_lor}
\mathrm{TR}(\state) \!=\! \scalebox{0.88}{$\bigcap\limits_{i=1}^{m}\!\clist{\youtput^* \Big | \|\nabla h_{i}(\state)\| \|A\!+\!BK\| \|\state \!-\! X \youtput^*\| \leq \alpha (h_i(\state)\!)\!}$}\!\!\!
\end{equation}
which satisfies $\mathrm{TR}(\state) \subseteq \barriercorridor_{\mathrm{lor}}(\state)$, by the Cauchy–Schwarz inequality.
The trust region $\mathrm{TR}(\state)$ allow for determining a persistent and safe output regulation set and for identifying the related requirements on control barrier functions, barrier decay rate, and control convergence rate.

\begin{proposition}\label{prop.safe_persistent_output_regulation}
\emph{(Safe and Persistent Output Regulation)}
Consider the closed-loop system in \eqref{eq:closed_loop_lor} and the set $\mathrm{TR}(\state) \subseteq \barriercorridor_{\mathrm{lor}}(\state)$ in \eqref{eq:trust_region_lor}. If the control barrier functions $h_i(\state)$ are convex and the barrier decay function $\alpha(.)$ satisfies
\begin{equation} \label{eq:alpha_lor}
    \alpha(h_i(\state)) \leq \frac{1}{2}\|A+BK\| h_i(\state),
\end{equation}
then any desired output $\youtput^* \in \mathrm{TR}(\state)$ is safely reachable, i.e., $h_i(\state(t)) \geq 0$ for all $i$ along the closed-loop trajectory $\state(t)$ of \eqref{eq:closed_loop_lor} for all $t \geq 0$.
\end{proposition}

\begin{proof}
Let $L := A+BK$, $\vect{z} := X\youtput^*\!-\state$, and $g_i := \nabla h_i(\state)$ to simplify the notation. Since $\youtput^* \!\in \mathrm{TR}(\state)$ by assumption and by the Cauchy–Schwarz inequality
\begin{equation} \label{eq:TR-CS}
\begin{aligned}
    0 &\leq \alpha(h_i(\state)) - \|L\| \|g_i\| \|\vect{z}\|\\
      &\leq \alpha(h_i(\state)) + \|L\| g_i^\top \vect{z}.
\end{aligned}
\end{equation}
The closed-loop trajectory of \eqref{eq:closed_loop_lor} with initial condition $\state$ at time $t = 0$ is:
\begin{equation}
    \state(t) = e^{Lt}\state + (I-e^{Lt})X\youtput^*, \quad t\geq 0.
\end{equation}
To relate \eqref{eq:TR-CS} to the closed-loop trajectory, we add and subtract $\|L\| g_i^\top e^{Lt}\vect{z}$, which satisfies $\|L\| g_i^\top e^{Lt}\vect{z} \leq \|L\| \|g_i\| \|e^{Lt}\| \|\vect{z}\| \leq \|L\|\|g_i\| \|\vect{z}\| \leq \alpha(h_i(\state))$ by the Cauchy–Schwarz inequality, $\|e^{Lt}\| \leq 1$ for all $t$, and \eqref{eq:trust_region_lor}. Thus, we obtain:
\begin{align*}
0 &\leq \alpha(h_i(\state)) + \|L\| g_i^\top \vect{z} - \|L\| g_i^\top e^{Lt}\vect{z} + \|L\| g_i^\top e^{Lt}\vect{z} \\
 &\leq \alpha(h_i(\state)) + \|L\| g_i^\top(I-e^{Lt})\vect{z} + \|L\| \|g_i\| \|\vect{z}\|\\
 &\leq 2 \alpha(h_i(\state)) + \|L\| g_i^\top (\state(t)-\state)\\
 &\leq \|L\|( h_i(\state) + g_i^\top(\state(t)-\state) )\\
 &\leq \|L\| h_i(\state(t)),
\end{align*}
where we use the choice of $\alpha$ in \eqref{eq:alpha_lor} in the second-to-last step and the convexity of $h_i$ in the last step. The last inequality implies $h_i(\state(t)) \geq 0$ for all $t \geq 0$.
\end{proof}

\begin{figure}[t]
\centering
\begin{tabular}{@{}c@{\hspace{1mm}}c@{\hspace{1mm}}c@{}}
\includegraphics[width=0.32\columnwidth, trim=10mm 10mm 10mm 10mm, clip]{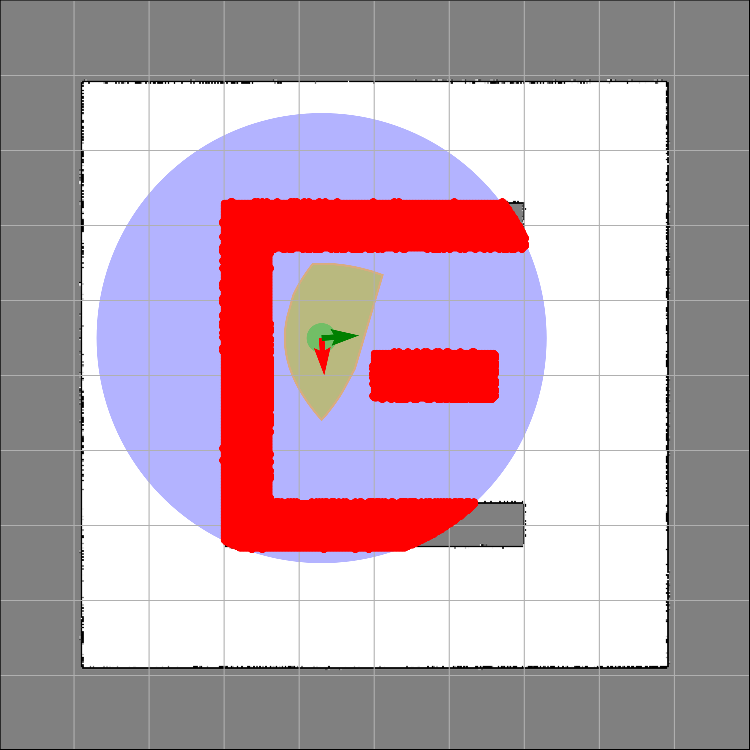} 
&
\includegraphics[width=0.32\columnwidth, trim=10mm 10mm 10mm 10mm, clip]{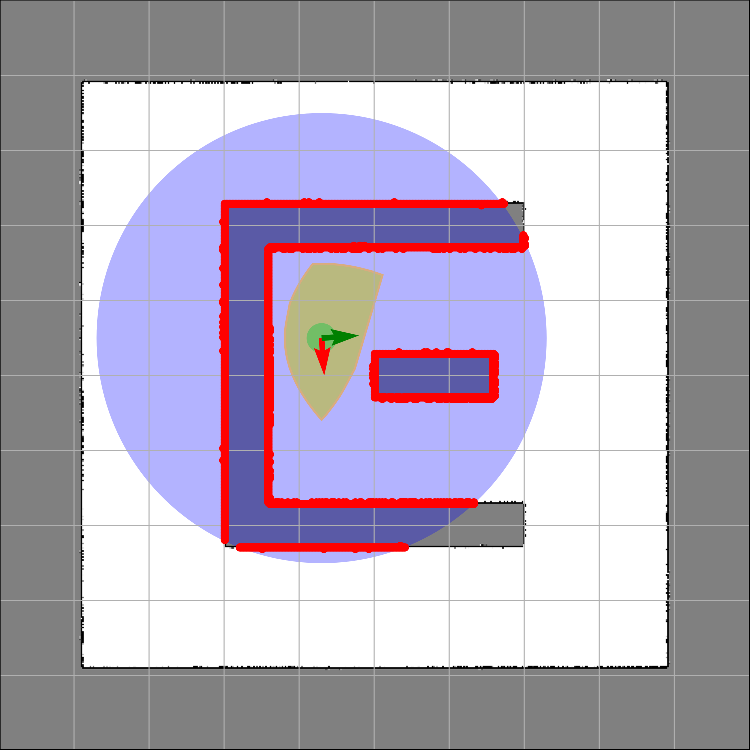} 
&
\includegraphics[width=0.32\columnwidth, trim=10mm 10mm 10mm 10mm, clip]{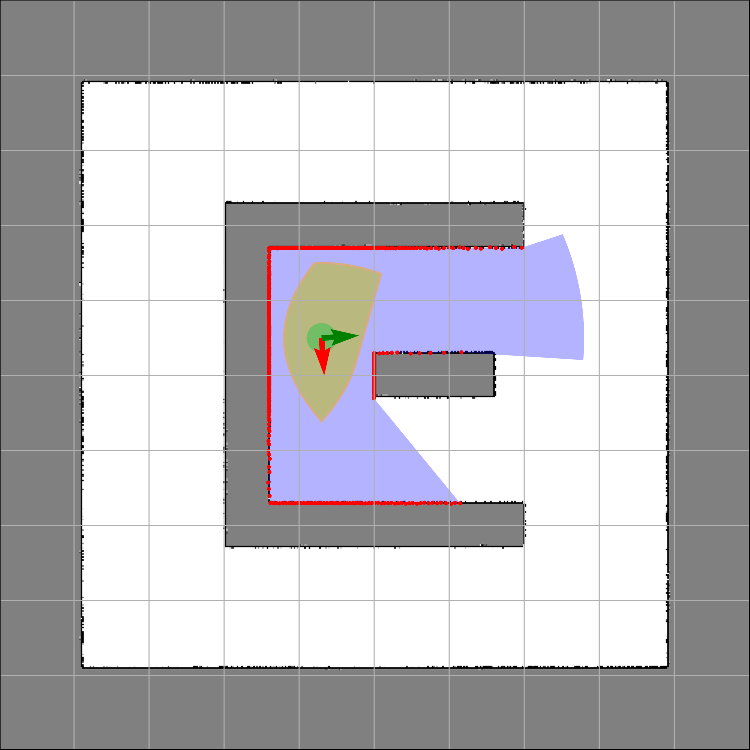} 
\end{tabular}
\vspace{-2mm}
\caption{Sensor-based control barrier corridors (yellow), constructed using (left) sensed unfree (red), i.e., occupied (black) and unknown(gray) locations in an occupancy grid map, (middle) the boundary (red) of the unfree space in the occupancy map, and (right) the sensed visible obstacle points from a laser scanner within a certain range around the robot, almost always result in the same barrier corridor, except for occlusions. Hence, visible obstacle points constitute the essential element for barrier safety and safe control.}
\label{fig.sensor_based_control_barrier_corridors}
\vspace{-3mm}
\end{figure}

\section{Application of Control Barrier Corridors: Safe and Persistent Path Following}
\label{sec.safe_path_following}

In this section, we present an application of control barrier corridors for local goal selection to ensure verifiable safe and persistent path following, and we demonstrate sensor-based path-following control for safe mobile robot navigation and exploration in unknown environments.

\subsection{Safe Path Following with Control Barrier Corridors}

A practical use of control barrier corridors is local goal selection for safe and persistent path following around obstacles, ensuring collision avoidance while maintaining continuous progress along the path. 
More precisely, given a reference (e.g., navigation) path $\refpath(s): [0,1] \rightarrow \R^{n}$, starting at $\refpath(0)$ and ending at $\refpath(1)$, we consider a local path goal selection strategy over a control barrier corridor $\barriercorridor(\state)$ of a closed-loop robotic system around a safe robot state $\state$ (with $h_{i}(\state) \geq i$ for all i) as selecting the farthest path point $\refpath(s)$ contained within the barrier corridor $\barriercorridor(\state)$, defined as
\begin{align}\label{eq.path_goal_selection}
\refpathgoal(\state):= \refpath\plist{\raisebox{3mm}{$\argmax\limits_{\substack{s \in [0,1]\\ \refpath(s) \in \barriercorridor(\state)}} s$}}.
\end{align}
For example, the control barrier corridor of a mobile robot can be constructed using the power distance to obstacle-occupied cells in an occupancy grid map of the environment, or to sensed visible obstacles within a certain sensing range, as illustrated in \reffig{fig.sensor_based_control_barrier_corridors}.
The local safety of control barrier corridors plays a critical role in ensuring the safe reachability of the selected path goal, as demonstrated in \reffig{fig.local_goal_selection_over_control_barrier_corridors}.
A matching pair of safety barrier decay rate and proportional control gain ($\barrierrate = \ctrlgain$), together with the convexity of control barrier functions (see \refprop{prop.geometry_meets_dynamics}), enables safe path goal selection within a safe control barrier corridor, thereby ensuring verifiable and reliable path following around obstacles.

\begin{figure}[t]
\centering
\begin{tabular}{@{}c@{\hspace{1mm}}c@{\hspace{1mm}}c@{}}
\includegraphics[width=0.32\columnwidth, trim=10mm 10mm 10mm 10mm, clip]{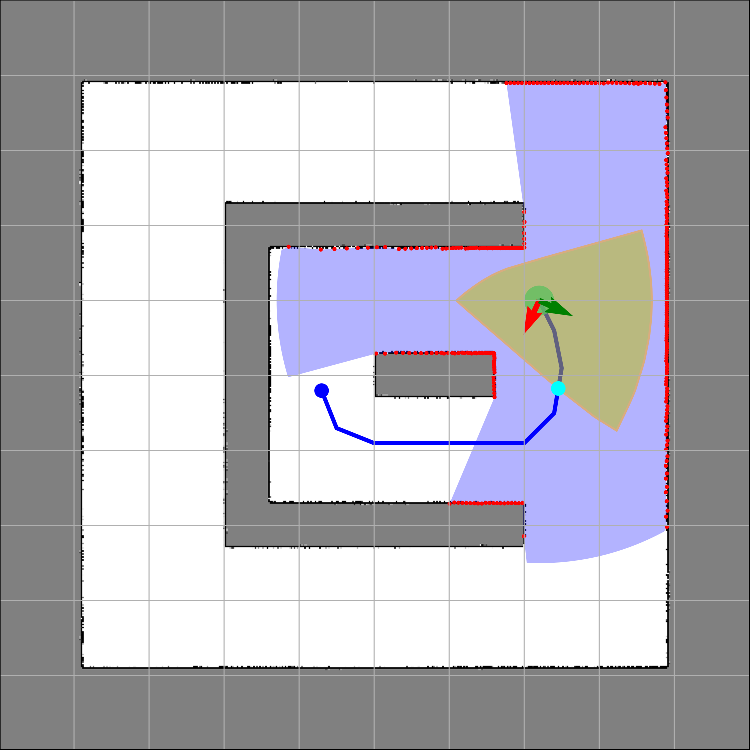} 
&
\includegraphics[width=0.32\columnwidth, trim=10mm 10mm 10mm 10mm, clip]{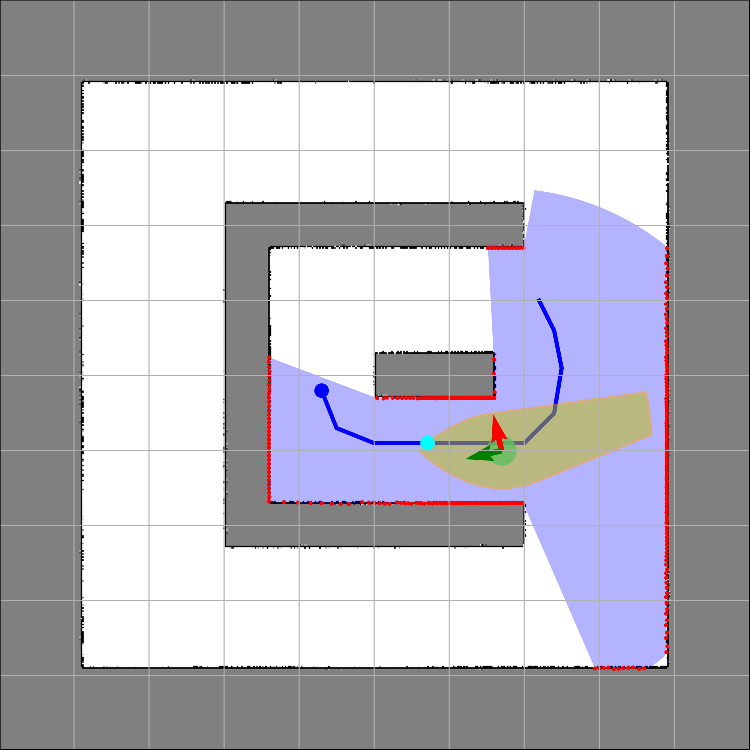} 
&
\includegraphics[width=0.32\columnwidth, trim=10mm 10mm 10mm 10mm, clip]{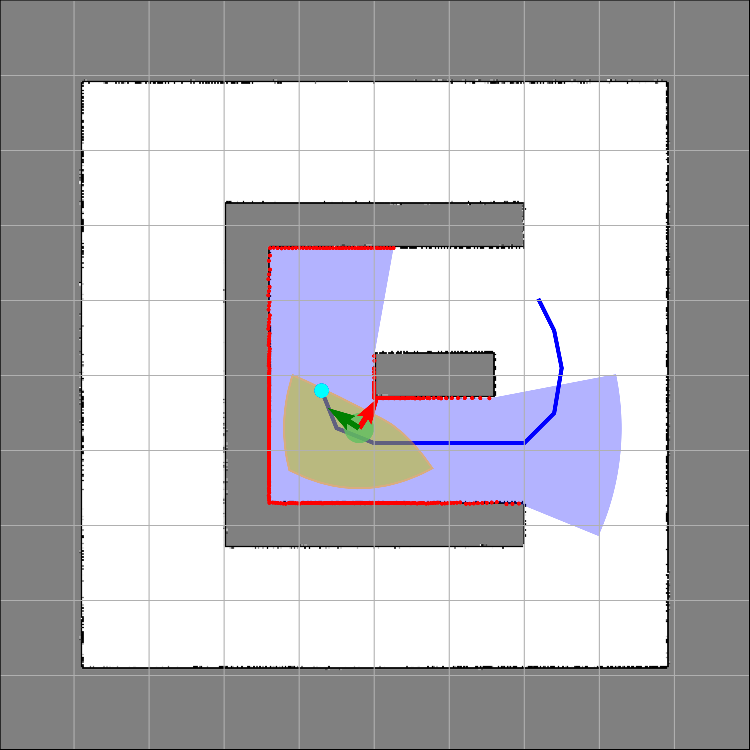} 
\\
\includegraphics[width=0.32\columnwidth, trim=10mm 10mm 10mm 10mm, clip]{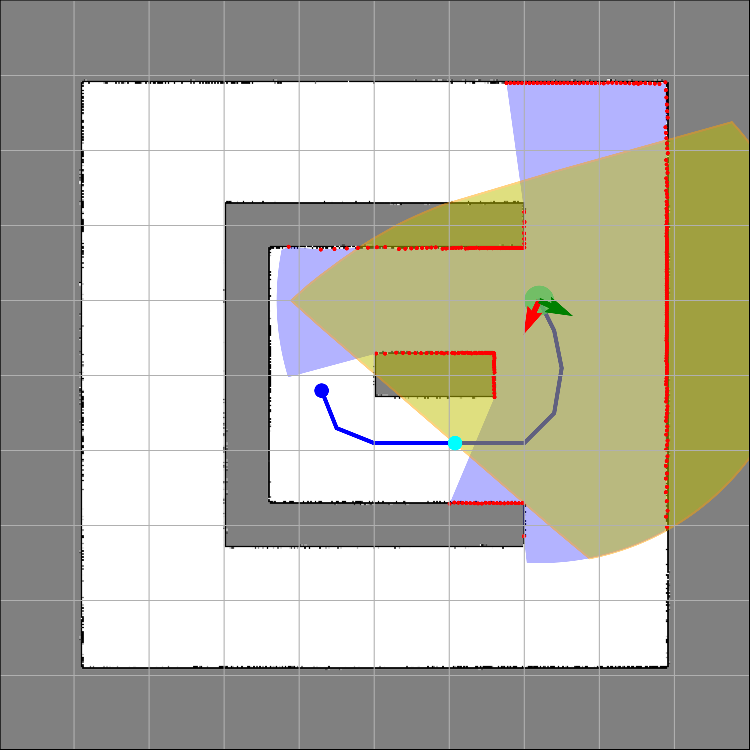} 
&
\includegraphics[width=0.32\columnwidth, trim=10mm 10mm 10mm 10mm, clip]{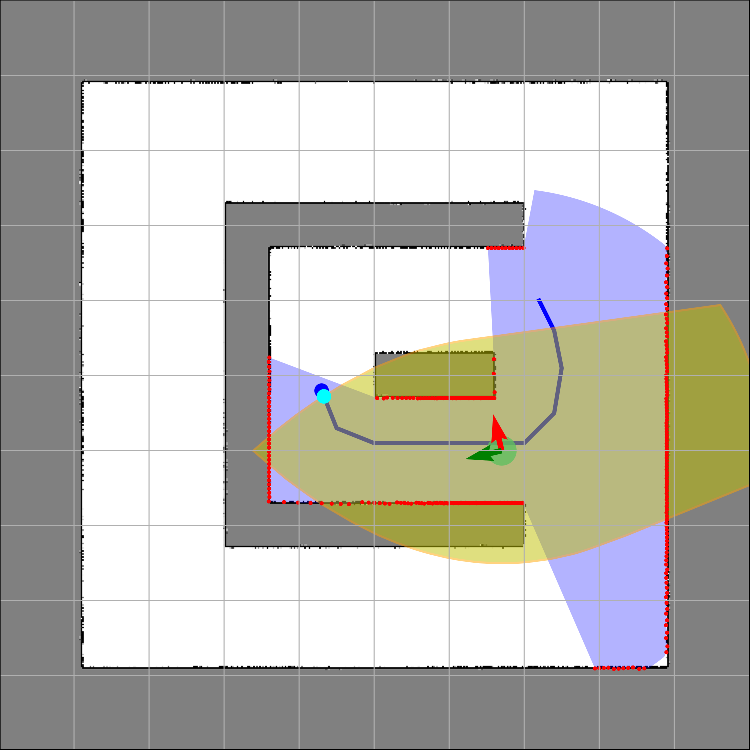} 
&
\includegraphics[width=0.32\columnwidth, trim=10mm 10mm 10mm 10mm, clip]{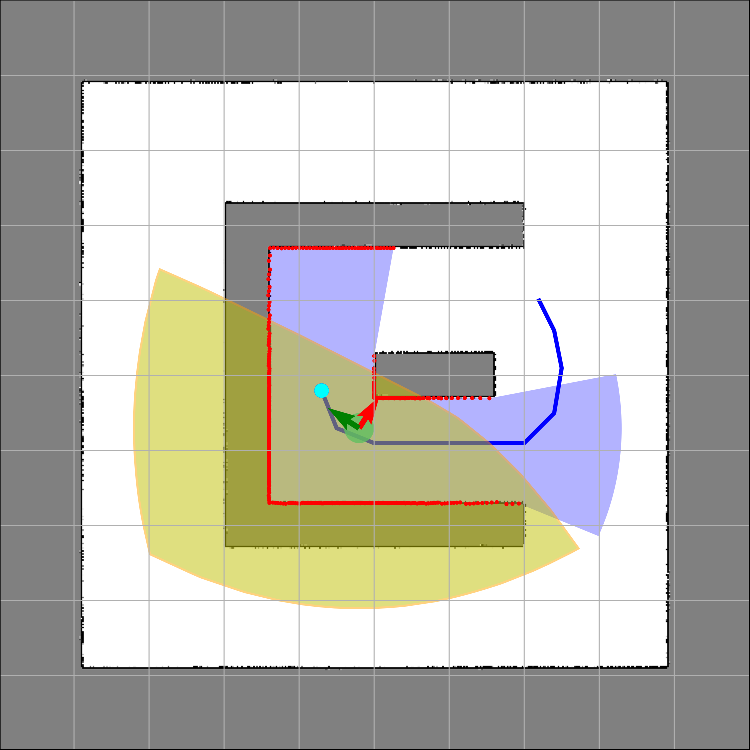} 
\end{tabular}
\vspace{-2mm}
\caption{Local goal selection (cyan) for following a path (blue) using control barrier corridors (yellow).  
Matching the barrier decay rate and the control convergence rate ($\barrierrate = \ctrlgain$, top) enables safe goal selection and reliable path-following control, as opposed to using a higher barrier decay rate than the control convergence rate ($\barrierrate = 3\ctrlgain$, bottom).}
\label{fig.local_goal_selection_over_control_barrier_corridors}
\vspace{-3mm}
\end{figure}

\begin{figure*}
\begin{tabular}{@{}c@{}}
\includegraphics[width=0.14\textwidth, trim=10mm 10mm 10mm 10mm, clip]{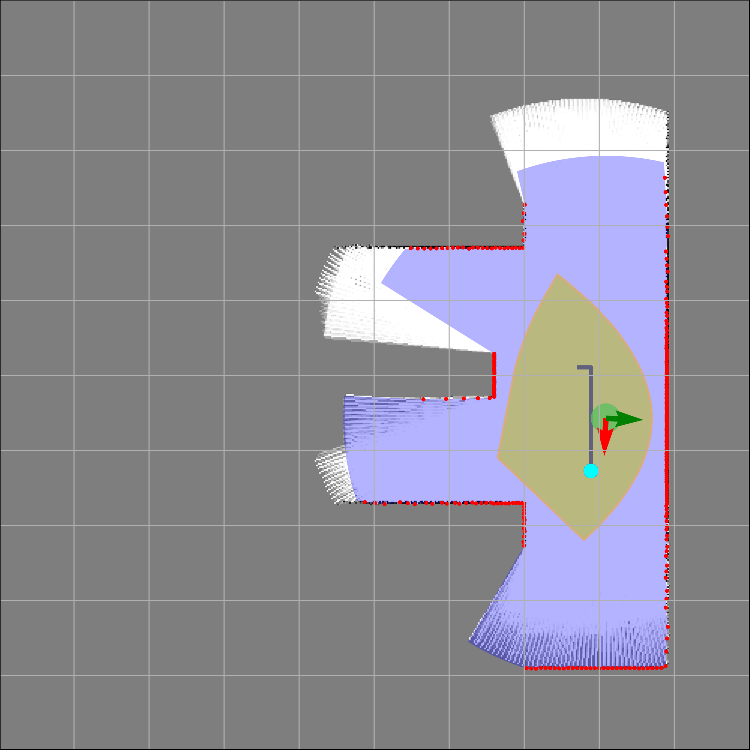} 
\hspace{-1.8mm}
\includegraphics[width=0.14\textwidth, trim=10mm 10mm 10mm 10mm, clip]{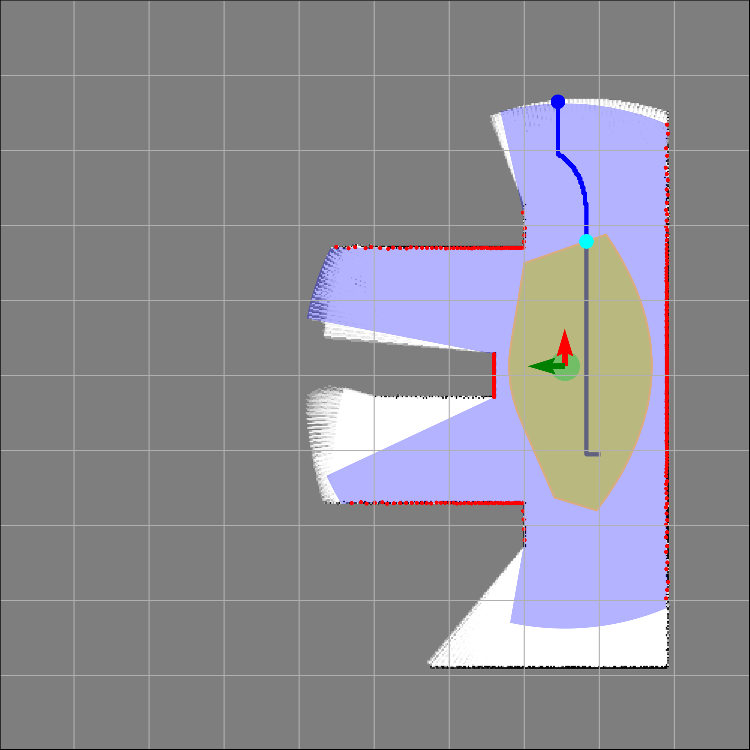} 
\hspace{-1.8mm}
\includegraphics[width=0.14\textwidth, trim=10mm 10mm 10mm 10mm, clip]{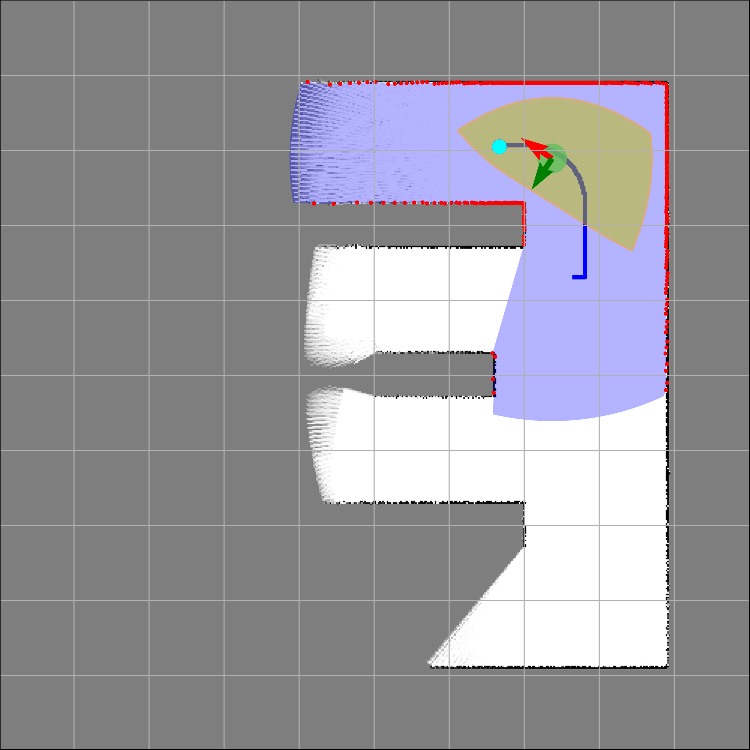} 
\hspace{-1.8mm}
\includegraphics[width=0.14\textwidth, trim=10mm 10mm 10mm 10mm, clip]{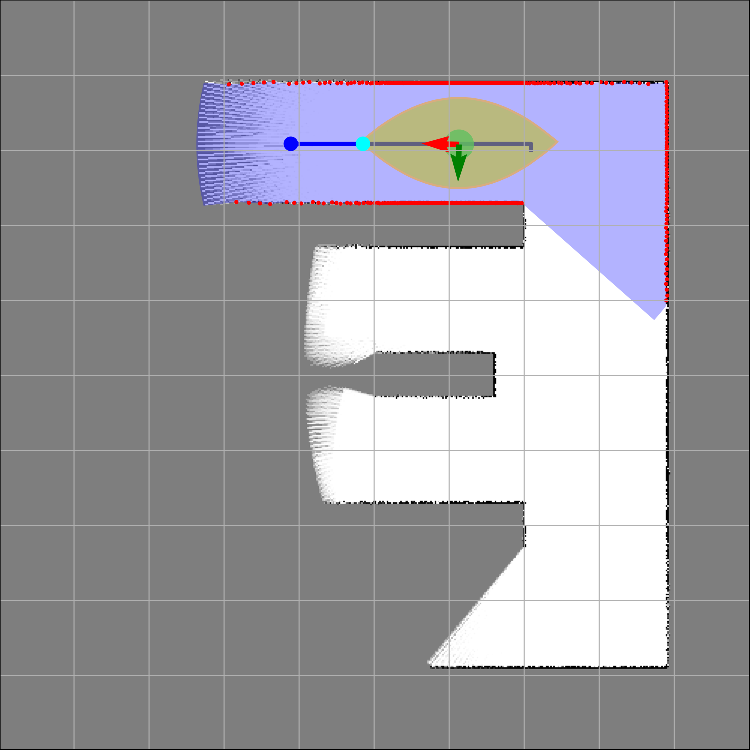} 
\hspace{-1.8mm}
\includegraphics[width=0.14\textwidth, trim=10mm 10mm 10mm 10mm, clip]{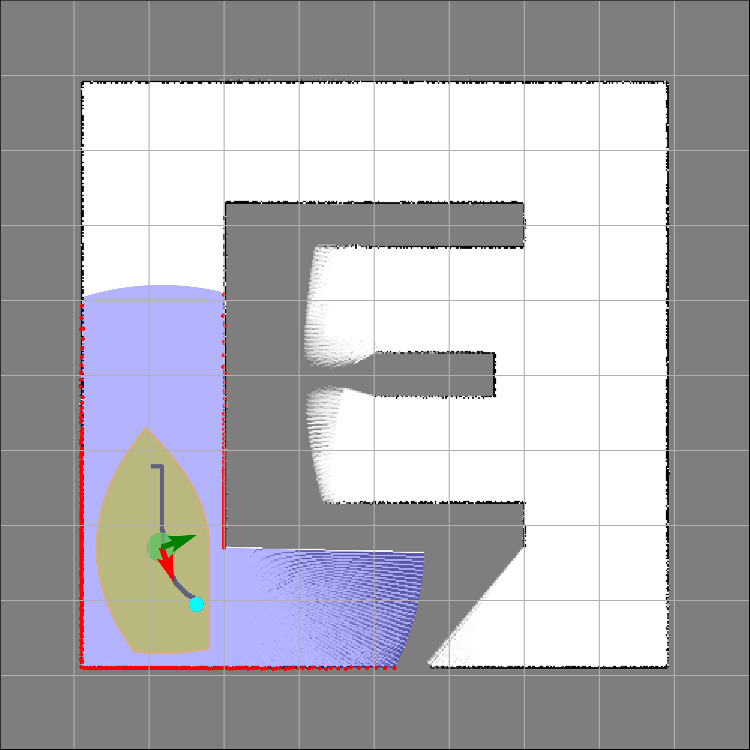} 
\hspace{-1.8mm}
\includegraphics[width=0.14\textwidth, trim=10mm 10mm 10mm 10mm, clip]{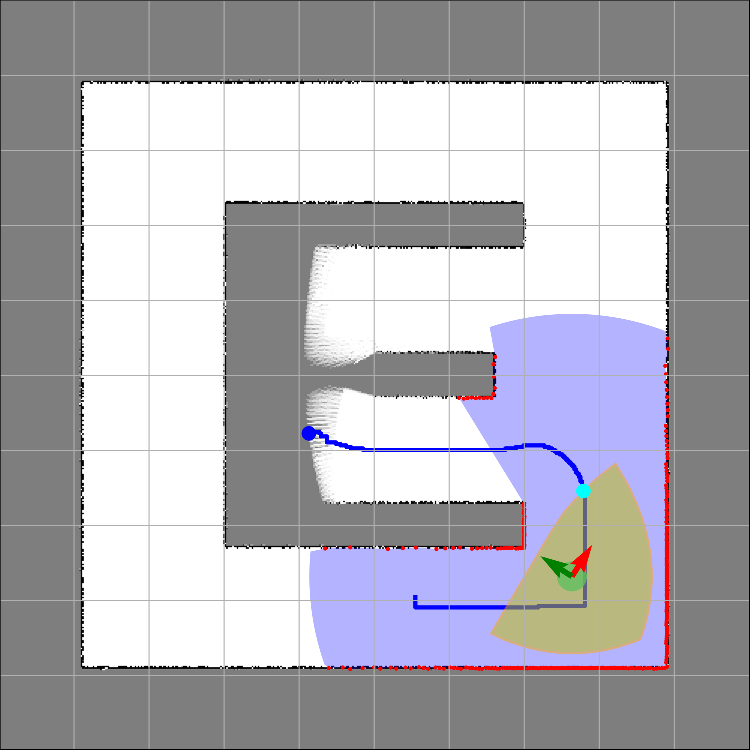} 
\hspace{-1.8mm}
\includegraphics[width=0.14\textwidth, trim=10mm 10mm 10mm 10mm, clip]{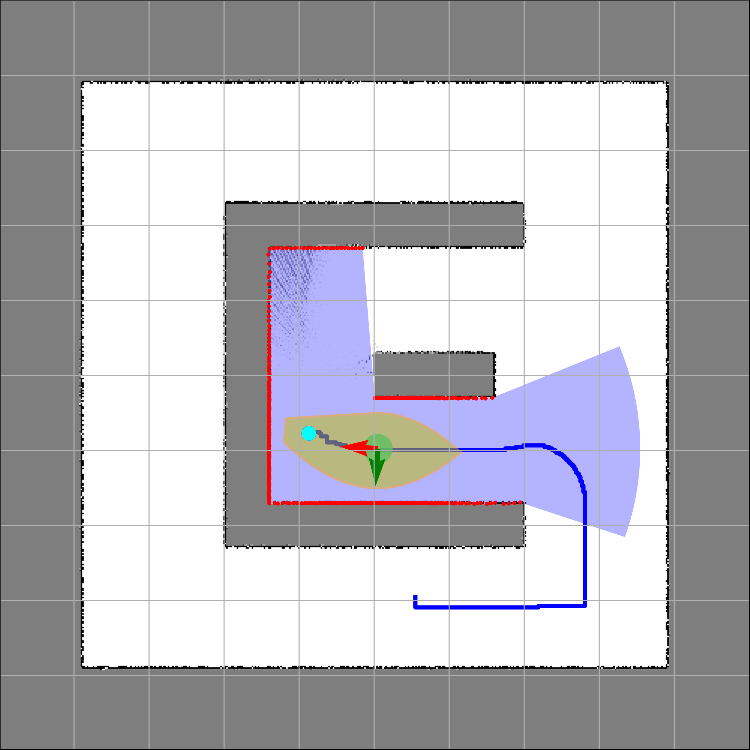} 
\end{tabular}
\vspace{-3mm}
\caption{Safe and persistent path following for frontier-based autonomous exploration of an unknown static environment, using a reference exploration path (blue line) toward a reachable frontier point (blue dot), with safely reachable path goal selection (cyan) over control barrier corridors (yellow) of a unicycle mobile robot (cyan + red arrow) constructed from sensed obstacle points (red) via an onboard laser scanner sensor.}
\label{fig.safe_exploration_control_barrier_corridor_unicycle}
\vspace{-3mm}
\end{figure*}

\begin{proposition} \label{prop.safe_persistent_path_following}
\emph{(Safe and Persistent Path Following with Control Barrier Corridors)}
Given a set of convex control barrier functions $h_1(\state), \ldots, h_{\nbarrier}(\state)$ for the fully actuated system $\dot{\state} = \ctrl$ and a strictly safe reference path $\refpath:[0,1] \rightarrow \R^{n}$ with $h_i(\refpath(s)) > 0$ for all $s \in [0,1]$ and $i=1, \ldots, \nbarrier$,  
for any matching barrier rate and control gain, $\barrierrate = \ctrlgain$, starting from any initial state $\state_0$ with $\refpath([0,1]) \cap \barriercorridor_{\mathrm{full}}(\state_0) \neq  \varnothing$, the continuous motion of the  fully actuated robot under proportional control 
\begin{align}
\dot{\state} = -\kappa \big(\state - \refpathgoal(\state)\big)
\end{align}  
towards the farthest path point $\refpathgoal(\state)$ of the reference path within the barrier corridor $\barriercorridor_{\mathrm{full}}(\state)$ asymptotically brings the robot to the end of the path, while ensuring robot safety along the motion trajectory $\state(t)$, the existence of a safely reachable path goal all times, and persistent progress along the reference path, 
i.e.,
\begin{align}
h(\state(t)) \geq 0 \quad &\forall \, t \geq 0, 
\\
\refpath([0,1]) \cap \barriercorridor(\state(t)) \neq \varnothing  \quad &\forall \, t \geq 0,
\\
\argmax_{\substack{s \in [0,1] \\ \refpath(s) \in \barriercorridor(\state(t_1))}}  \hspace{-2mm}s
  \, \leq  \hspace{-2mm}
  \argmax_{\substack{s \in [0,1] \\ \refpath(s) \in \barriercorridor(\state(t_2))}}\hspace{-2mm} s
   \quad &\forall \, 0 \leq t_1 \leq t_2 ,
 \\
 \lim_{t \rightarrow \infty} \state(t)  = \refpath(1). \quad &
\end{align}  
\end{proposition}
\begin{proof}
The safety of the robot along the trajectory $\state(t)$ follows from the definition in \refeq{eq.full_control_barrier_corridor}, which states that any goal $\goal \in \barriercorridor_{\mathrm{full}}(\state(t))$ ensures safety feasibility under the proportional goal control $\ctrl = -\ctrlgain(\state - \goal)$.
The existence of a safely reachable path goal at all times is a direct consequence of the safe and persistent goal property of the full control barrier corridor $\barriercorridor_{\mathrm{full}}(\state)$ in \refprop{prop.geometry_meets_dynamics} and \refprop{prop.persistent_safe_goal_control}.
The persistent goal property in \refprop{prop.persistent_safe_goal_control} also ensures persistent progress along the reference path, since the previously selected goal is instantaneously maintained within the control barrier corridor $\barriercorridor_{\mathrm{full}}(\state)$, while the farthest-path goal selection in \refeq{eq.path_goal_selection} may either keep it or replace it with a farther point if available.
Finally, since all path points are strictly safe (i.e., $h_i(\refpath(s)) > 0$), the control barrier corridor around any path point $\barriercorridor_{\mathrm{full}}(\refpath(s))$ has a nonempty interior with nonzero measure, ensuring that the robot can remain at a path point only for a finite time, except when asymptotically converging to the endpoint $\refpath(1)$ over time.
\end{proof}
Note that a similar conclusion also holds for the kinematic unicycle robot model under safe and persistent goal control mentioned in \refprop{prop.persistent_safe_unicycle_navigation} as long as the reference path is $\varepsilon$-safer with some extra safety margin $\varepsilon > 0$, as defined in \refeq{eq.safer_control_barrier_functions}, to allow smooth and continuous control near the goal.

\subsection{Sensor-Based Path Following for Mobile Exploration}

To demonstrate the applicability of control barrier corridors to sensor-based safe path following in mobile robot navigation, we consider frontier-based autonomous exploration of unknown environments \cite{yamauchi_CIRA1997}. 
We use a simulated mobile robot in the Gazebo simulator with unicycle dynamics and a 2D laser scanner for obstacle sensing, operating in a 2D environment with known position and onboard occupancy mapping capabilities \cite{elfes_C1989}. 
We construct the control barrier corridors using the Euclidean distance to sensed obstacle points (i.e., the order $p=1$ of the power distance) with a matching pair of linear velocity gain and barrier decay rate, $\barrierrate = \lingain$.  
The robot performs frontier-based exploration based on the following sequential perception, planning, and control cycle:
\begin{enumerate}[i)]
\item \label{item.frontier_selection} Goal Frontier Selection: Choose a reachable frontier grid farthest from occupied cells (a simple measure of information utility), located between known-free and unknown regions of the current occupancy grid map.
\item Safe Path Planning: Determine an optimal reference path (e.g., via graph search \cite{lavalle_PlanningAlgorithms2006}) using inverse distance to obstacles as local cost, minimizing path length while maximizing clearance.    
\item Safe Path Following: Select the farthest path point within the sensor-based control barrier corridor as a local navigation goal, and navigate with persistent unicycle goal control (\refprop{prop.persistent_safe_unicycle_navigation}) until the path endpoint is reached or no longer reachable; otherwise, go to \ref{item.frontier_selection}).
\end{enumerate}

In \reffig{fig.safe_exploration_control_barrier_corridor_unicycle}, we present examples of selected frontiers, reference paths, local path goals, and control barrier corridors during autonomous exploration of an unknown environment.
As observed in \reffig{fig.safe_exploration_control_barrier_corridor_unicycle}, convex control barrier corridors are adaptively constructed and shaped as local collision-free zones around the robot using sensed obstacles, and they continuously evolve with the continuous robot motion. 
Once the farthest path point is selected as a safe control goal for the unicycle, the unicycle control not only ensures the robot safety but also maintains the goal within the safe corridor for persistent progress along the exploration path (\refprop{prop.persistent_safe_unicycle_navigation}). 
Since safe and persistent path following only requires a safe reference path intersecting with the robot’s control barrier corridor (\refprop{prop.safe_persistent_path_following}), this enables replanning and updating the exploration path without coming to a full stop.  

\section{Conclusions}
\label{sec.conclusions}


In this paper, we present a generic framework for transforming control barrier functions of feedback control systems into control barrier corridors, converting functional safety requirements on control into geometric safety constraints on reference control goals. 
Example control barrier constructions for fully actuated systems under proportional control, kinematic unicycle systems under geometric inner-outer control, and linear systems under output regulation control show that individual state safety, encoded by control barrier functions, can be extended locally over control barrier corridors, enabling safe and persistent goal selection. 
As a result, safely reachable goal selection over control barrier corridors enables the design of verifiably safe, provably correct, and persistent path following in unknown environments using sensor-based control barrier corridors.

We are currently exploring the generalization of control barrier corridors to higher-order systems using higher-order control barrier functions. 
Another promising research question is how to systematically leverage geometric safe motion corridors and control convergence rates to learn  (neural) control barrier approximations with adaptive decay profiles.


%


%


\bibliographystyle{IEEEtran}
\bibliography{references.bib}


%
%
%

\end{document}